\documentclass[runningheads]{llncs}

\usepackage{eccv}

\usepackage{eccvabbrv}

\usepackage{graphicx}
\usepackage{booktabs}
\usepackage{hyperref}
\usepackage{multirow}

\usepackage[accsupp]{axessibility}  

\usepackage{orcidlink}

\begin{document}

\title{Two-Parameter Flow Map Learning for Continuous-Time Diffeomorphic Image Registration} 

\titlerunning{TPFM Diffeomorphic Image Registration}

\author{Mohammadjavad Matinkia\orcidlink{0009-0007-2326-5852} \and
Nilanjan Ray\orcidlink{0000-0002-7588-5400}}

\authorrunning{M.~Matinkia, N.~Ray}

\institute{University of Alberta, Canada\\ \email{matinkia@ualberta.ca}, \email{nray1@ualberta.ca}}

\maketitle

\begin{abstract}
  Diffeomorphic image registration is central to medical image analysis, enabling anatomically consistent alignment across subjects. Most learning-based diffeomorphic methods model autonomous ODEs (ordinary differential equations) by parameterizing a stationary velocity field and recovering deformations via scaling-and-squaring. While non-auto\-nomous ODEs with time-dependent velocities increase expressiveness, existing approaches rely on numerical integration to implicitly enforce flow structure that entangles model expressiveness with discretization accuracy. We propose a framework to directly learn the continuous-time solution of a non-autonomous ODE formulated as a two-parameter flow map. By enforcing cocycle consistency, a fundamental structural property of time-varying flows, we learn the flow maps without time discretization and velocity integration during training. The framework recovers diffeomorphic mappings at inference using a small number of compositions. Our proposed framework seamlessly incorporates standard registration backbones and improves alignment accuracy consistently across nine datasets while preserving diffeomorphic structure. Notably, the proposed method achieves an average Dice improvement of 2.1\% on brain MRI benchmarks, a 12\% TRE reduction on lung CT, and a 2.6\% Dice gain on cardiac MRI and ultrasound datasets (\url{https://mattkia.github.io/TPFMDIR/}).
\end{abstract}

\begin{figure}
    \centering
    \includegraphics[width=0.9\linewidth]{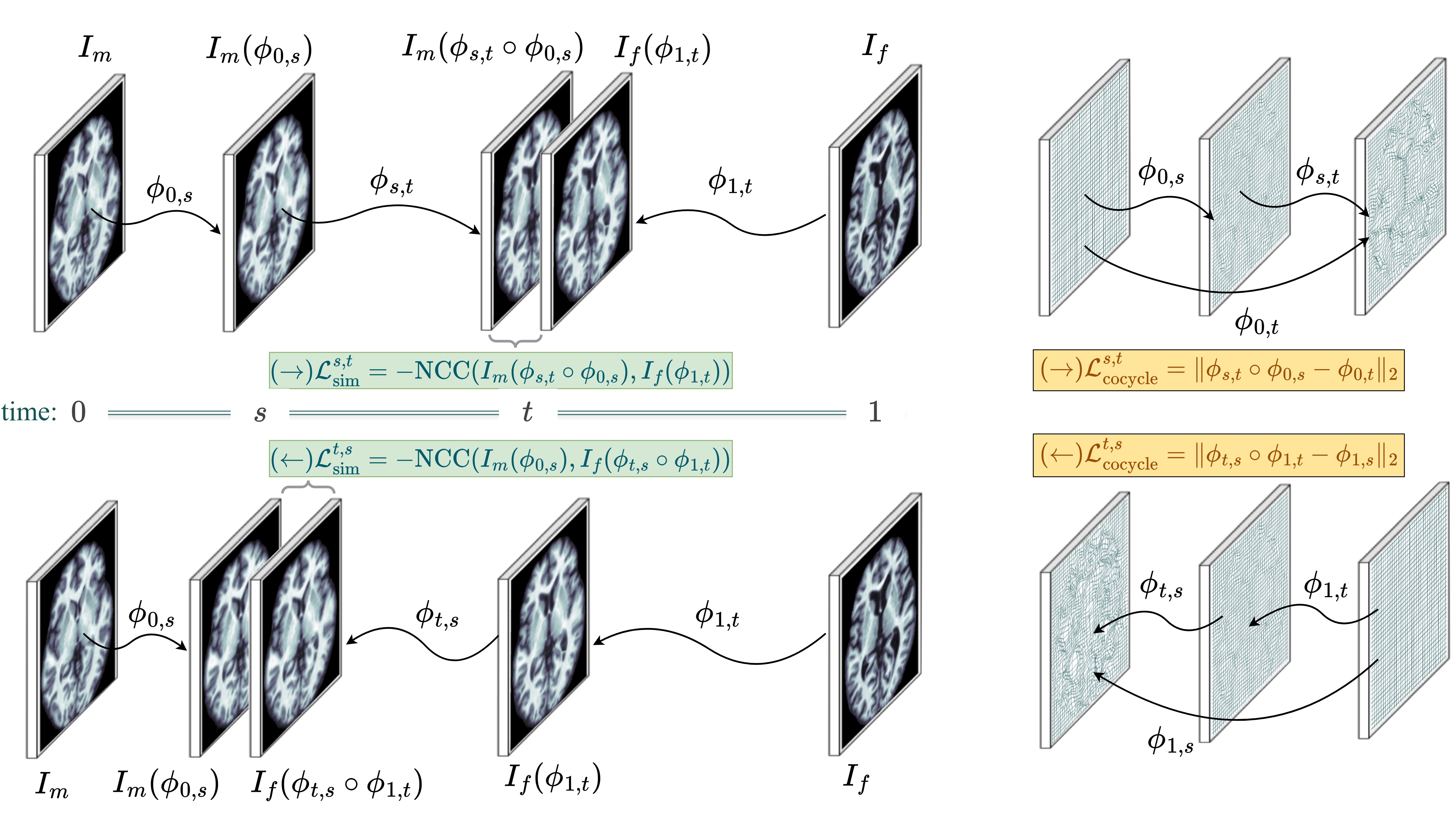}
    \caption{Overview of TPFM-DIR. The left panel illustrates time-dependent similarity supervision, while the right panel depicts cocycle consistency regularization. Top: forward-time training; bottom: backward-time training.}
    \label{fig:pitch}
\end{figure}

\section{Introduction}
\label{sec:intro}
Diffeomorphic image registration (DIR) aims to estimate a smooth, bijective transformation that warps a moving image onto a fixed image while preserving topology. Unlike general deformable registration methods, which prioritize image similarity under smoothness constraints, diffeomorphic approaches explicitly enforce smooth invertibility to prevent foldings and ensure anatomically plausible alignment. Although this constraint provides strong geometric guarantees, diffeomorphic methods have historically faced a trade-off between anatomical consistency and alignment accuracy \cite{tmsc}.

Traditional DIR methods rely on variational optimization to compute a deformation for each image pair, resulting in high computational cost \cite{syn,niftyreg,lddmm}. Neural approaches address this limitation by either optimizing network parameters per image pair or learning deformation priors across datasets for rapid inference\cite{learning-based1,learning-based2,learnin-based5}. A particularly influential class of neural DIR methods models deformations as flows generated by ordinary differential equations (ODEs)\cite{symnet,voxelmorph,transmorph}. In this formulation, the transformation is obtained by integrating a velocity field over time. Autonomous ODEs, being based on stationary velocity fields, allow efficient integration via scaling-and-squaring \cite{ss}. Non-autonomous formulations generalize this setting by allowing time-dependent velocity fields. However, they require costly numerical integration during both training and inference to induce flow structure. As a result, deformation consistency becomes coupled to discretization schemes and solver accuracy \cite{nodeo}.

We introduce \textbf{TPFM-DIR}, a flow-based framework, which directly models the continuous-time solution operator of a non-autonomous ODE. By doing so, the method eliminates explicit velocity parameterization and numerical integration during training. Our approach is grounded in the theoretical properties of non-autonomous ODE solutions. In particular, we exploit the \textbf{cocycle property}, which characterizes valid flow maps of non-autonomous systems \cite{nonauto}. We propose a cocycle-based regularization for flow maps. This regularization enforces temporal consistency, and promotes invertibility and regularity of the deformation. As a result, TPFM-DIR does not require explicit integration and multiple handcrafted smoothness penalties that are commonly used in diffeomorphic learning frameworks.

\begin{figure}[!t]
    \centering
    \includegraphics[width=.8\textwidth]{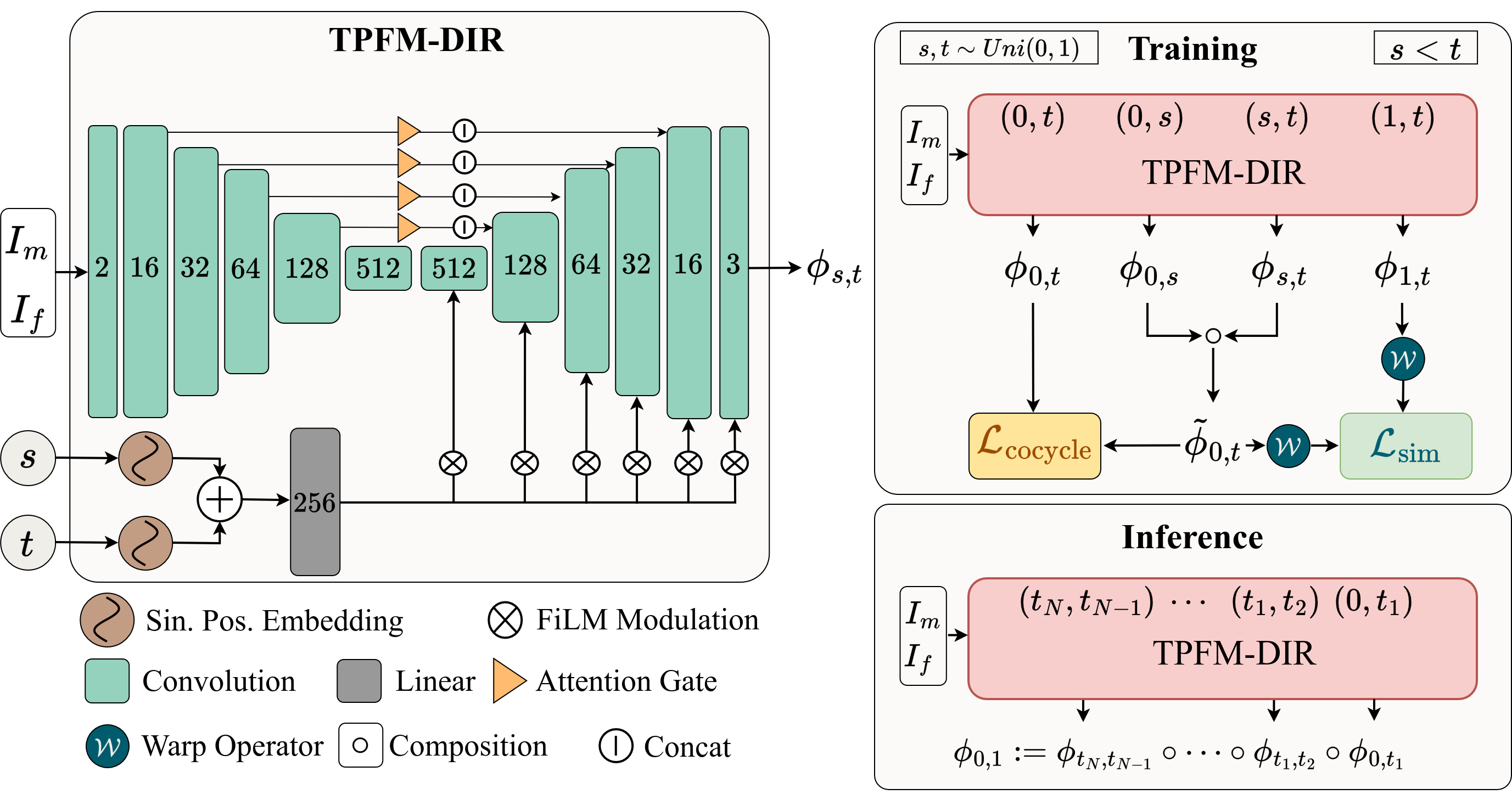}
    \caption{Architecture of TPFM-DIR (left), training scheme (top right), and inference scheme (bottom right). The training diagram illustrates the forward direction; the backward direction follows the formulation shown in Fig.~\ref{fig:pitch}. In all experiments, we use $N=4$ compositions during inference.}
    \label{fig:arch}
\end{figure}

By modeling the flow map directly, the framework captures complex time-dependent deformations. At inference time, global transformations are constructed through a small number of compositions of predicted short-time maps. This strategy improves deformation regularity while maintaining computational efficiency.

Empirically, TPFM-DIR narrows the performance gap between diffeomorphic and deformable registration, matching or surpassing strong deformable baselines while preserving topology. We further demonstrate that TPFM-DIR integrates with diverse backbone architectures, consistently improving both alignment accuracy and topology preservation. The contributions of this work are summarized as follows:

\begin{enumerate}
\item We introduce TPFM-DIR, a flow-map-based formulation of diffeomorphic image registration, that directly models the two-parameter solution of a non-autonomous ODE. Unlike velocity-based parameterizations that learn the generating field, our formulation operates at the level of the flow map itself, enabling structural modeling of the time-dependent deformations.
\item We propose a cocycle-based regularization with formal guarantees that characterize valid solutions of non-autonomous ODEs. This removes the need for explicit numerical integration during training while promoting temporal consistency and invertibility without relying on multiple handcrafted smoothness penalties.
\item We validate TPFM-DIR on nine benchmark datasets spanning 2D and 3D MRI, CT, and ultrasound, demonstrating consistent improvements over strong deformable and diffeomorphic methods.
\end{enumerate}

\section{Related Works}
\label{sec:related}

\paragraph{Deformable vs. Diffeomorphic Registration.} 
Deformable registration methods typically predict dense displacement fields and prioritize alignment accuracy. Recent performance gains are largely driven by increasingly expressive architectures, including vision transformers \cite{transmorph,transmatch,hvit} and correlation or structure aware modules \cite{corrmlp,sacbnet}. However, such models generally lack explicit structural constraints, which may result in folding under large deformations. To ensure anatomical plausibility, Diffeomorphic Image Registration (DIR) seeks smooth, invertible transformations \cite{survey1, survey2}. Classical geodesic formulations (e.g., LDDMM \cite{lddmm}, SyN \cite{syn}) provide strong theoretical guarantees but at high computational cost. Modern learning-based DIR approaches amortize this optimization across datasets \cite{learning-based1,learning-based2,learning-based3,learning-based4,learnin-based5}, typically enforcing regularity through architectural symmetry \cite{symnet} or surrogate penalties such as cycle, inverse, and gradient consistency \cite{cyclemorph,consistency1,consistency2,gradicon,nephi}. These constraints may trade registration accuracy for regularity \cite{survey2} and are usually defined only between image pairs.

\begin{figure}[!t]
    \centering
    \includegraphics[width=\linewidth]{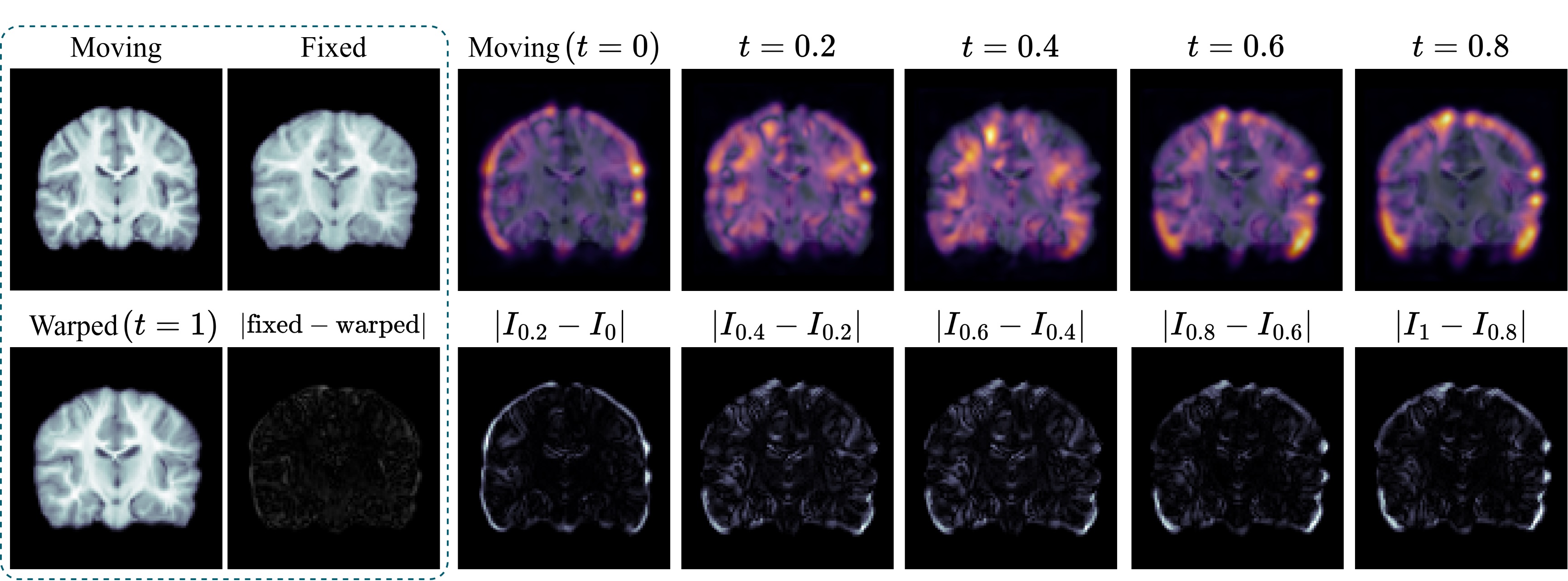}
    \caption{Continuous registration trajectory on CANDI. Top: warped image evolution with overlaid velocity norms computed by \cref{eq:vel} indicating time-varying velocity. Bottom: incremental image differences illustrating temporal smoothness.}
    \label{fig:candi}
\end{figure}

\paragraph{Flow-Based Diffeomorphic Models.} 
Flow-based DIR represents transformations as ODE-generated flows \cite{flow1,flow2,ant}. Autonomous formulations assume stationary velocity fields and enable efficient scaling-and-squaring integration to recover the deformations \cite{symnet,pulpo,voxelmorph,lapirn,liemorph}. However, the stationarity assumption restricts the admissible deformation dynamics, whereas non-autonomo\-us models allow time-varying velocities and provide a more general deformation framework. Classical LDDMM \cite{lddmm} and neural extensions such as NODEO \cite{nodeo} and R2Net \cite{r2net} fall into this category. The increased expressiveness of non-autonomous flows comes at both computational and structural complexity. Existing approaches rely on explicit numerical integration \cite{nodeo,MSODE}, coupling diffeomorphic guarantees to discretization accuracy and solver design. In contrast, our approach directly models the two-parameter flow family, eliminating explicit velocity integration while preserving the structural foundations of flow-based diffeomorphisms.

\section{Proposed Method}
\label{sec:method}

\begin{figure}[!t]
    \centering
    \includegraphics[width=\linewidth]{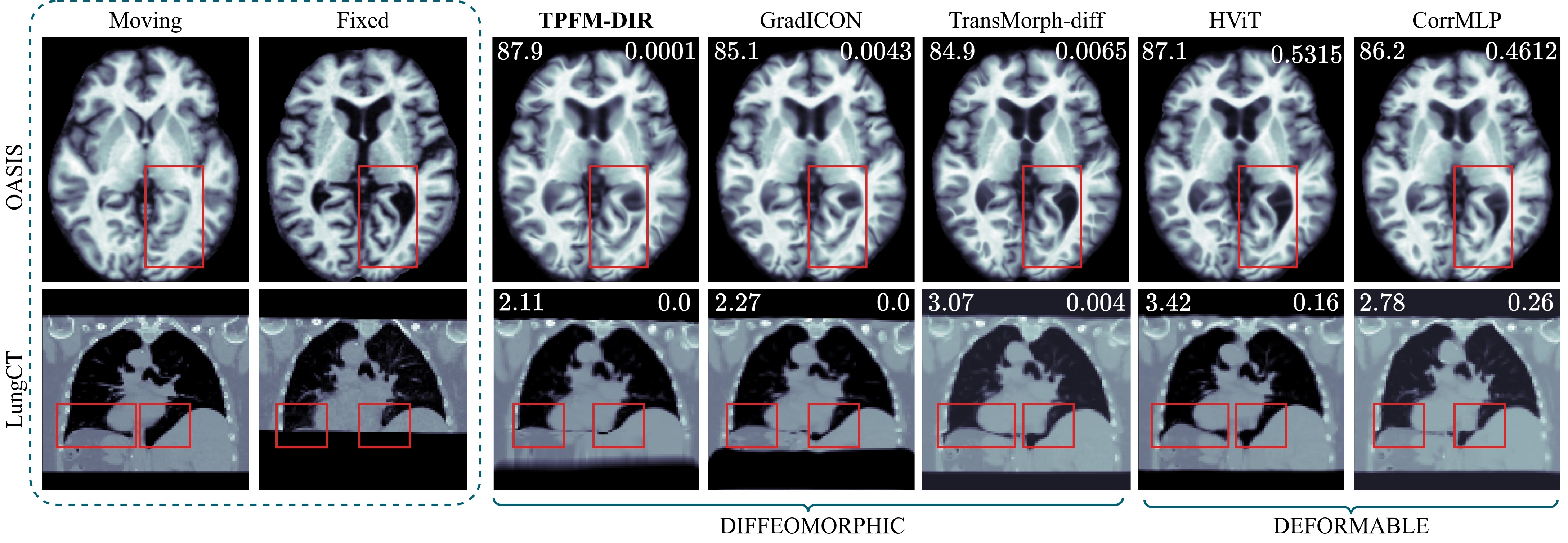}
    \caption{Visual comparisons on OASIS and LungCT. Dice/TRE are shown at top-left corners and $|J|_{<0}\%$ is shown at top-right corners.}
    \label{fig:oasis}
\end{figure}

\subsection{Preliminaries}
\label{sec:prem}

Let $I_m, I_f: \Omega\rightarrow\mathbb{R}$ denote the moving and fixed images defined over a $d$-dimensional spatial domain $\Omega\subset\mathbb{R}^d$. Diffeomorphic registration seeks a transformation $\phi: \Omega\rightarrow\Omega$ that balances image similarity and transformation regularity. To obtain diffeomorphic transformations, many methods model $\phi$ as the flow of an ODE under suitable regularity conditions. A common approach utilizes an \textit{autonomous} ODE \cite{symnet,transmorph}

\begin{gather} 
    \label{eq:aode} 
    \frac{d\phi_t}{dt} = v(\phi_t), \quad \phi_0 = Id, 
\end{gather}

where $v$ is a stationary velocity field, $Id$ denotes the identity mapping, and $\phi_t$ is the deformation at time $t$. This formulation is favored for its computational efficiency via the scaling-and-squaring integration scheme \cite{ss}. On the other hand, non-autonomous ODEs generalize the autonomous case by adopting time-dependent velocity fields \cite{lddmm,nodeo}:

\begin{gather} 
    \label{eq:naode}
    \frac{d\phi_t}{dt} = v(\phi_t, t), \quad \phi_0 = Id.
\end{gather}

While more expressive, non-autonomous ODEs rely on numerical integration during training, coupling deformation structure to time discretization and solver accuracy.

In this work, we shift the focus from the velocity field $v$ to the direct modeling of the \textbf{two-parameter flow map} $\phi_{s,t}:\Omega\rightarrow\Omega$, which represents the transformation from time $s$ to time $t$ and encodes the solution of a non-autonomous ODE, eliminating the need for discretized velocity integration to construct the flow. For a differentiable mapping $\phi_{s,t}$ to be a valid solution of a non-autonomous ODE, it must satisfy the identity condition and the \textbf{cocycle property} \cite{nonauto}:

\begin{gather} 
    \label{eq:cocycle}
    \phi_{s,s} = Id, \quad \ \phi_{r,t} \circ \phi_{s,r} = \phi_{s,t}, \quad \forall s \leq r \leq t \in [0, 1].
\end{gather}

The cocycle property dictates that the mapping from time $s$ to $t$ is equivalent to the composition of intermediate mappings through any intermediate time $r$. Leveraging this structural constraint, we learn $\phi_{s,t}$ directly without explicit velocity parameterization, decoupling flow modeling from velocity integration. In the rest of the paper, the moving image is associated with time $0$, and the fixed image with time $1$.

\subsection{Two-Parameter Flow Maps For Diffeomorphic Image Registration}
\label{sec:tpfm}


The core objective of TPFM-DIR is to directly model the flow map associated with the non-autonomous ODE in \cref{eq:naode}, represented as a two-parameter family of transformations. We parameterize the flow map $\phi_{s, t}$ as

\begin{gather}
    \label{eq:phi_model}
    \phi_{s, t} = Id + (t - s) f_\theta(s, t; I_m, I_f),
\end{gather}

where $f_\theta(.)$ is a time-dependent neural network with parameters $\theta$, conditioned on the moving and fixed images and time parameters $s$ and $t$. This construction enforces $\phi_{s,s}=Id$ and defines all transformations in a shared spatial frame, enabling valid compositions. This parameterization alone does not guarantee a valid flow map; the cocycle regularization imposes the necessary structural constraints to promote consistency with a continuous-time flow. We implement $f_\theta$ using a doubly time-embedded U-Net (\cref{fig:arch}) with sinusoidal embeddings for time samples $s$ and $t$ \cite{meanflow,flowmatching}. Smoothness and invertibility are further promoted through the shared-coordinate formulation and cocycle regularization.

To train the model without numerical integration, we propose a temporal similarity loss and an \textit{anchored} cocycle regularization with fixed endpoints. Following the conventional approaches of time-dependent architectures, we sample $s,t \sim Uni[0,1]$ with $s\leq t$ \cite{meanflow}. To ensure unbiased learning and memory efficiency, we employ a symmetric training scheme that alternates which branch carries the composed transformation at each epoch. This avoids doubled memory usage while preserving symmetry. We define the similarity loss as

\begin{equation}
    \label{eq:sim-loss}
    \mathcal{L}_{\text{sim}}^{s, t} =
    \begin{cases}
        -NCC\big(I_m(\phi_{s, t} \circ \phi_{0, s}),\ I_f(\phi_{1,t})\big) & \text{even epochs} \\
        -NCC\big(I_m(\phi_{0, s}),\ I_f(\phi_{t, s} \circ \phi_{1,t})\big) & \text{odd epochs}
    \end{cases}
\end{equation}

This aligns the forward trajectory of the moving image with the backward trajectory of the fixed image at time $t$. To enforce the structural consistency of the flow, we introduce the \textit{anchored} cocycle regularization:

\begin{equation}
    \label{eq:reg}
    \mathcal{L}_{\text{cocycle}}^{s, t} =
    \begin{cases}
        \|\phi_{s, t} \circ \phi_{0, s} - \phi_{0, t}\|_2 & \text{even epochs} \\
        \|\phi_{t, s} \circ \phi_{1, t} - \phi_{1, s}\|_2 & \text{odd epochs} 
    \end{cases}
\end{equation}

\Cref{fig:pitch} illustrates the training process. In practice, compositions are implemented via spatial warping with linear interpolation. Minimizing \cref{eq:reg} enforces the anchored cocycle constraints:

\begin{gather}
    \label{eq:partial-coc}
    \phi_{s, t} \circ \phi_{0, s} = \phi_{0, t},\quad \phi_{t, s} \circ \phi_{1, t} = \phi_{1, s}
\end{gather}

Hence, the total loss becomes 

\begin{gather}
    \label{eq:loss}
    \mathcal{L}^{s, t} = \mathcal{L}_{\text{sim}}^{s, t} + \lambda \mathcal{L}_{\text{cocycle}}^{s, t}
\end{gather}

where $\lambda$ balances the trajectory consistency. In our experiments we show that a value $\lambda=10$ leads to the most promising results. Note that enforcing the full cocycle constraint for arbitrary triplets $(r, s, t)$ would require evaluating multiple nested compositions per iteration and is substantially more expensive. We therefore enforce only the anchored cocycle constraints in \cref{eq:partial-coc}, which are computationally more efficient while remaining sufficient in practice:

\begin{theorem}
    \label{thm}
    Under standard regularity assumptions, a deformation $\phi_{s,t}$ satisfying the identity condition and the anchored cocycle properties of \cref{eq:partial-coc} is a two-parameter diffeomorphism solving a non-autonomous ODE of \cref{eq:naode} on a bounded open domain $\Omega$.
\end{theorem}
\begin{proof}
    Appendix A.
\end{proof}

An immediate consequence of \cref{eq:phi_model} is that the instantaneous velocity admits a closed-form expression without numerical differentiation:

\begin{corollary}
    \label{corr}
    For the flow map $\phi_{s,t}$ parameterized as \cref{eq:phi_model}, the instantaneous velocity can be analytically obtained via
    \begin{gather}
        \label{eq:vel}
        v_t = f_\theta(t, t; I_m, I_f).
    \end{gather}
\end{corollary}
\begin{proof}
    Appendix A.
\end{proof}

This allows direct extraction of the velocity fields for further investigations (see \cref{fig:candi}). The complete training and inference schema can be viewed in \cref{fig:pitch,fig:arch}.

Uniform sampling of $s$ and $t$ biases the time-difference distribution toward small increments, as shorter intervals occur more frequently under uniform sampling. Consequently, large transitions are relatively underrepresented during optimization, which may lead to mild deviations from ideal flow structure when directly predicting $\phi_{0,1}$ (see Appendix B for a complete discussion). Thus, although $\phi_{0,1}$ can be predicted in a single pass at inference, we partition the interval into $N$ subintervals and compose locally predicted maps to better approximate a continuous-time evolution at inference:

\begin{equation}
    \label{eq:integrate}
    \phi_{0, 1} = \phi_{t_N,1} \circ \dots \circ \phi_{0,t_1},
\end{equation}

Each segment corresponds to a well-trained local deformation, improving stability and reducing the risk of foldings. In our experiments, we find that a $N=4$ compositions suffices for constructing an accurate diffeomorphic deformation. 

\section{Experimental Results}
\label{sec:exp}
Unless otherwise specified, all experiments use $\lambda=10$ (\cref{eq:loss}) and inference uses 4 compositions (\cref{eq:integrate}). TPFM-DIR is optimized with AdamW with learning rate $10^{-4}$ and batch size of 1. Training is performed for 200 epochs on all datasets except LungCT (100 epochs). We set a windows size of 7 for the NCC loss across all experiments. All experiments are conducted on a single NVIDIA RTX 3090 GPU (24GB VRAM). Additional implementation details are provided in Appendix D.

\begin{table}[!t]
    \centering
    \caption{Quantitative comparison on the OASIS and IXI datasets. The best result per metric is shown in \textbf{bold}. The second-best diffeomorphic and deformable results are highlighted in \textcolor{blue}{blue} and \textcolor{red}{red}, respectively.}
    \resizebox{\textwidth}{!}{
    \begin{tabular}{lcccccccc}
        \toprule
        & \multicolumn{4}{c}{\textbf{OASIS}} & \multicolumn{4}{c}{\textbf{IXI}} \\
        \cmidrule(lr){2-5} \cmidrule(lr){6-9}
        \textbf{Method} 
        & \textbf{Dice $\uparrow$} 
        & $|J|_{<0}\% \downarrow$ 
        & \textbf{HD95} $\downarrow$ 
        & \textbf{SDLogJ} $\downarrow$ 
        & \textbf{Dice $\uparrow$} 
        & $|J|_{<0}\% \downarrow$ 
        & \textbf{HD95} $\downarrow$ 
        & \textbf{SDLogJ} $\downarrow$ \\
        \midrule
        LDDMM & 76.59 $\pm$ 2.42 & 0.0064 $\pm$ 0.0051 & 3.89 $\pm$ 0.93 & 0.012 $\pm$ 0.009 & 67.90 $\pm$ 1.15 & 0.0034 $\pm$ 0.0021 & 4.11 $\pm$ 1.02 & 0.058 $\pm$ 0.020 \\
        \midrule
        CycleMorph & 81.93 $\pm$ 2.14 & 0.0211 $\pm$ 0.0091 & 2.36 $\pm$ 0.81 & 0.040 $\pm$ 0.026 & 74.67 $\pm$ 1.33 & 0.0224 $\pm$ 0.0110 & 3.34 $\pm$ 0.87 & 0.068 $\pm$ 0.032 \\
        GradICON & \textcolor{blue}{83.74 $\pm$ 1.42} & 0.0039 $\pm$ 0.0012 & \textcolor{blue}{2.09 $\pm$ 0.36} & 0.011 $\pm$ 0.006 & \textcolor{blue}{76.43 $\pm$ 1.28} & \textcolor{blue}{0.0018 $\pm$ 0.0022} & \textcolor{blue}{3.19 $\pm$ 0.55} & \textcolor{blue}{0.015 $\pm$ 0.006} \\
        TransMorph-diff & 83.51 $\pm$ 1.52 & 0.0066 $\pm$ 0.0073 & 2.35 $\pm$ 0.76 & 0.071 $\pm$ 0.041 & 75.98 $\pm$ 2.01 & 0.0142 $\pm$ 0.0098 & 3.81 $\pm$ 0.89 & 0.102 $\pm$ 0.048 \\
        R2Net & 80.11 $\pm$ 1.53 & 0.0093 $\pm$ 0.0054 & 2.15 $\pm$ 0.49 & 0.081 $\pm$ 0.055 & 73.39 $\pm$ 1.77 & 0.0031 $\pm$ 0.0016 & 3.69 $\pm$ 0.66 & 0.042 $\pm$ 0.026 \\
        NODEO & 81.97 $\pm$ 1.39 & \textcolor{blue}{0.0024 $\pm$ 0.0010} & 2.18 $\pm$ 0.71 & \textcolor{blue}{0.008 $\pm$ 0.005} & 75.68 $\pm$ 1.82 & 0.0031 $\pm$ 0.0015 & 3.46 $\pm$ 0.84 & 0.017 $\pm$ 0.008 \\
        NePhi & 80.09 $\pm$ 1.17 & 0.0056 $\pm$ 0.0023 & 3.28 $\pm$ 0.44 & 0.022 $\pm$ 0.016 & 75.34 $\pm$ 1.78 & 0.0029 $\pm$ 0.0016 & 3.55 $\pm$ 0.55 & 0.018 $\pm$ 0.011 \\
        PULPo & 80.51 $\pm$ 0.99 & 0.0527 $\pm$ 0.0240 & 3.45 $\pm$ 0.98 & 0.077 $\pm$ 0.031 & 73.91 $\pm$ 2.11 & 0.0066 $\pm$ 0.0043 & 3.30 $\pm$ 0.39 & 0.028 $\pm$ 0.016 \\
        \midrule
        TransMorph & 84.11 $\pm$ 1.30 & 1.0665 $\pm$ 0.5631 & 2.26 $\pm$ 0.68 & 0.821 $\pm$ 0.252 & 77.72 $\pm$ 1.76 & 1.2574 $\pm$ 0.7574 & 3.71 $\pm$ 0.80 & 1.083 $\pm$ 0.742 \\
        TransMatch & 83.48 $\pm$ 1.80 & 0.1493 $\pm$ 0.0151 & 2.23 $\pm$ 0.66 & 0.741 $\pm$ 0.344 & 76.80 $\pm$ 1.85 & 0.0406 $\pm$ 0.0120 & 3.15 $\pm$ 0.41 & 0.580 $\pm$ 0.411 \\
        DiffuseMorph & 79.13 $\pm$ 2.03 & 1.1563 $\pm$ 0.2512 & 3.51 $\pm$ 0.77 & 0.336 $\pm$ 0.189 & 71.35 $\pm$ 2.05 & 1.6042 $\pm$ 0.7212 & 5.48 $\pm$ 1.45 & 0.938 $\pm$ 0.508 \\
        DiffuseReg & 80.12 $\pm$ 1.22 & 1.0877 $\pm$ 0.4312 & 3.73 $\pm$ 1.01 & 0.377 $\pm$ 0.130 & 72.13 $\pm$ 1.88 & 0.5279 $\pm$ 0.2834 & 5.04 $\pm$ 1.12 & 0.488 $\pm$ 0.276 \\
        HViT & \textcolor{red}{85.07 $\pm$ 1.05} & 0.4812 $\pm$ 0.0114 & \textcolor{red}{1.92 $\pm$ 0.46} & 0.631 $\pm$ 0.412 & \textcolor{red}{80.67 $\pm$ 1.67} & 0.5933 $\pm$ 0.1028 & \textcolor{red}{2.98 $\pm$ 0.46} & 0.679 $\pm$ 0.290 \\
        CorrMLP & 84.66 $\pm$ 1.24 & 0.4640 $\pm$ 0.1776 & 2.23 $\pm$ 0.62 & 0.608 $\pm$ 0.278 & 77.54 $\pm$ 1.68 & 0.3675 $\pm$ 0.2168 & 3.15 $\pm$ 0.39 & 0.547 $\pm$ 0.231 \\
        SACB-Net & 83.17 $\pm$ 1.18 & \textcolor{red}{0.0184 $\pm$ 0.0093} & 2.33 $\pm$ 0.48 & \textcolor{red}{0.182 $\pm$ 0.087} & 78.15 $\pm$ 1.87 & \textcolor{red}{0.0201 $\pm$ 0.0099} & 3.19 $\pm$ 0.51 & \textcolor{red}{0.144 $\pm$ 0.056} \\
        DGIR & 82.37 $\pm$ 1.33 & 0.0325 $\pm$ 0.0076 & 3.29 $\pm$ 0.91 & 0.392 $\pm$ 0.182 & 75.34 $\pm$ 1.55 & 0.0766 $\pm$ 0.0101 & 3.22 $\pm$ 0.72 & 0.193 $\pm$ 0.101 \\
        \midrule
        \textbf{TPFM-DIR} & \textbf{87.08 $\pm$ 1.07} & \textbf{0.0019 $\pm$ 0.0007} & \textbf{1.78 $\pm$ 0.44} & \textbf{0.006 $\pm$ 0.002} & \textbf{82.52 $\pm$ 0.98} & \textbf{0.0015 $\pm$ 0.0006} & \textbf{2.95 $\pm$ 0.68} & \textbf{0.007 $\pm$ 0.003} \\
        \bottomrule
    \end{tabular}}
    \label{tab:oasis-ixi}
\end{table}

\subsection{Datasets and Metrics}
\label{sec:data}

We evaluate TPFM-DIR on nine datasets spanning multiple modalities, anatomical regions, and spatial dimensions. Brain MRI benchmarks include OASIS~\cite{oasis}, IXI\footnote{\url{https://brain-development.org/ixi-dataset}}, Mindboggle101~\cite{mindboggle}, LPBA40~\cite{lpba}, and CANDI~\cite{candi}. For thoracic and abdominal CT, we use LungCT and AbdomenCT from the Learn2Reg challenge\footnote{\url{https://learn2reg.grand-challenge.org/}}, representing large respiratory motion and multi-organ abdominal alignment, respectively. To assess 2D registration, we use ACDC~\cite{acdc} (cardiac MRI) and CAMUS~\cite{camus} (cardiac ultrasound) datasets.

For datasets with segmentation masks, we report the Dice similarity coefficient as the primary overlap metric. Surface alignment is evaluated using the 95th percentile Hausdorff Distance (HD95) and the Average Symmetric Surface Distance (ASSD). For LungCT, where expert landmarks are available, we compute Target Registration Error (TRE), and additionally report Structural Similarity Index (SSIM) to assess intensity consistency. Diffeomorphic and smoothness properties are evaluated via percentage of voxels with negative Jacobian determinant ($|J|_{<0}\%$) and the standard deviation of the logarithm of the Jacobian determinant (SDLogJ). To enhance readability of the tables, we only report the main metrics for each dataset, and present the tables with complete metrics in Appendix F.

\begin{figure}[!t]
    \centering
    \includegraphics[width=\linewidth]{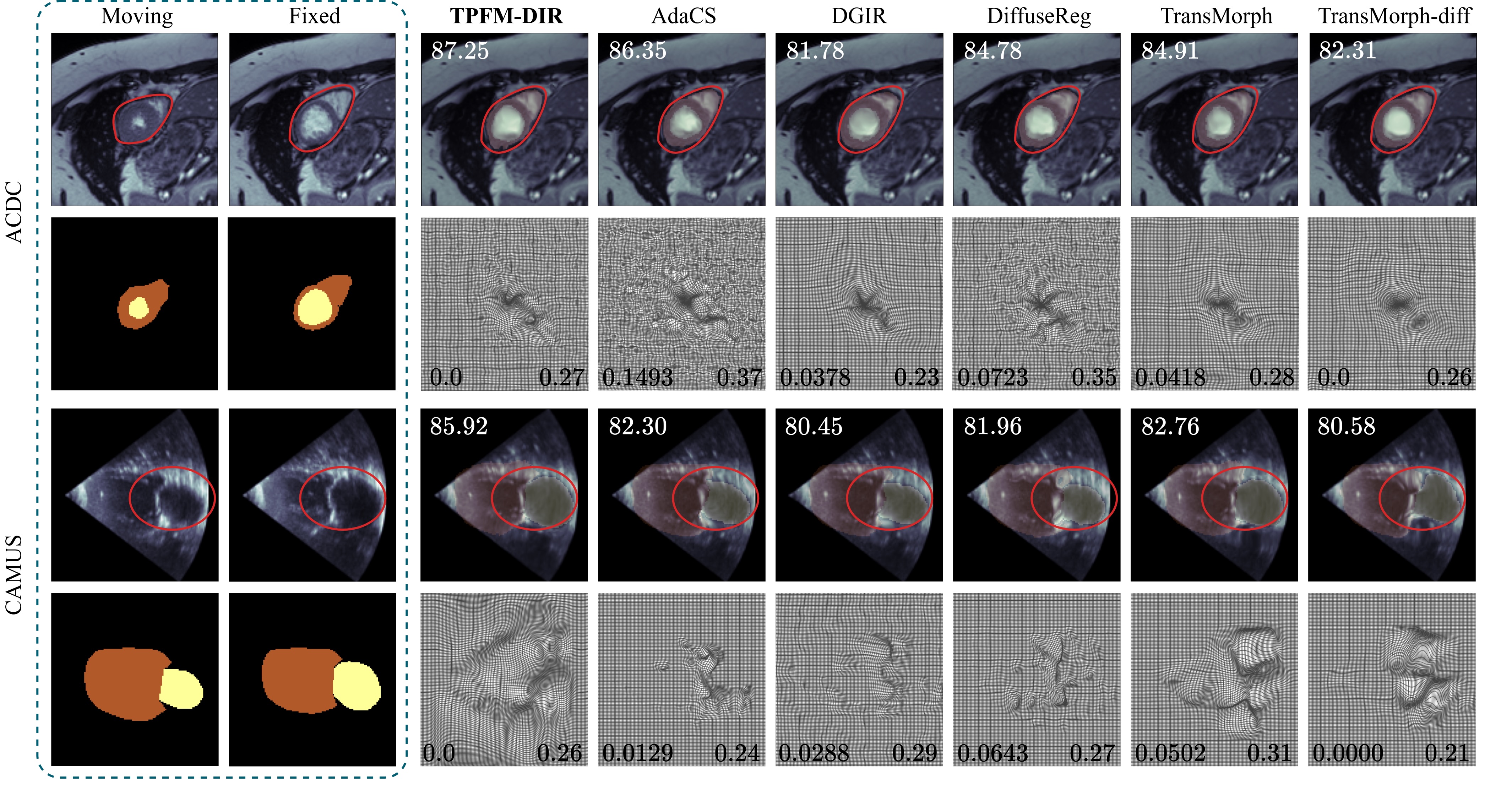}
    \caption{Qualitative comparisons on ACDC and CAMUS. Warped images with overlaid segmentations are shown with Dice reported at the top-left corners; $|J|_{<0}\%$ and SDLogJ are shown at the bottom-left and bottom-right of the deformation grids, respectively.}
    \label{fig:camus}
\end{figure}

Our evaluations include both inter-subject (brain MRI and AbdomenCT) and intra-subject (LungCT and cardiac sequences) registration. For inter-subject experiments, subject pairs are randomly sampled following a fixed protocol. Complete dataset splits and preprocessing are detailed in Appendix C.

\subsection{Results}
\label{sec:results}
Across all datasets, we compare TPFM-DIR against a broad range of state-of-the-art methods. The diffeomorphic baselines include CycleMorph, GradICON, TransMorph-diff, R2Net, NODEO, NePhi, and PULPo \cite{cyclemorph,gradicon,transmorph,r2net,nodeo,nephi,pulpo}. Deformable baselines include TransMorph, TransMatch, DiffuseMorph, DiffuseReg, HViT, CorrMLP, SACB-Net, and DGIR \cite{transmatch,transmorph,diffusemorph,diffusereg,hvit,corrmlp,sacbnet,dgir}. Among the diffeomorphic methods, R2Net and NODEO adopt non-autonomous ODE formulations. Moreover, we include comparisons with LDDMM \cite{lddmm} as a classical diffeomorphic method based on non-autonomous ODEs. For 2D datasets, we evaluate only methods with official 2D implementations and additionally include AdaCS as a strong baseline for ACDC and CAMUS \cite{adacs}.

\begin{table}[!t]
    \centering
    \caption{Quantitative comparison on the LPBA40, Mindboggle101, and CANDI datasets. The best result per metric is shown in \textbf{bold}. The second-best diffeomorphic and deformable results are highlighted in \textcolor{blue}{blue} and \textcolor{red}{red}, respectively.}
    \resizebox{\textwidth}{!}{
    \begin{tabular}{lccccccccc}
        \toprule
        & \multicolumn{3}{c}{\textbf{LPBA40}} 
        & \multicolumn{3}{c}{\textbf{Mindboggle101}} 
        & \multicolumn{3}{c}{\textbf{CANDI}} \\
        \cmidrule(lr){2-4} \cmidrule(lr){5-7} \cmidrule(lr){8-10}
        \textbf{Method} 
        & \textbf{Dice $\uparrow$} 
        & $|J|_{<0}\% \downarrow$ 
        & \textbf{HD95} $\downarrow$ 
        & \textbf{Dice $\uparrow$} 
        & $|J|_{<0}\% \downarrow$ 
        & \textbf{HD95} $\downarrow$ 
        & \textbf{Dice $\uparrow$} 
        & $|J|_{<0}\% \downarrow$ 
        & \textbf{HD95} $\downarrow$ \\
        \midrule
        LDDMM & 69.23 $\pm$ 1.78 & 0.0011 $\pm$ 0.0004 & 7.38 $\pm$ 1.11 & 64.57 $\pm$ 2.08 & 0.0061 $\pm$ 0.0039 & 7.79 $\pm$ 0.84 & 78.52 $\pm$ 1.66 & 0.0017 $\pm$ 0.0007 & 2.44 $\pm$ 0.31 \\
        \midrule
        CycleMorph & 71.90 $\pm$ 1.41 & 0.0064 $\pm$ 0.0032 & 6.76 $\pm$ 0.94 & 70.64 $\pm$ 1.48 & 0.1282 $\pm$ 0.0873 & 7.75 $\pm$ 0.98 & 81.37 $\pm$ 1.25 & 0.0027 $\pm$ 0.0013 & 2.23 $\pm$ 0.29 \\
        GradICON & 74.21 $\pm$ 1.26 & 0.0011 $\pm$ 0.0006 & 6.89 $\pm$ 0.89 & \textcolor{blue}{70.78 $\pm$ 1.52} & 0.0117 $\pm$ 0.0043 & 7.59 $\pm$ 1.12 & \textcolor{blue}{83.39 $\pm$ 1.09} & 0.0015 $\pm$ 0.0008 & 2.14 $\pm$ 0.27 \\
        TransMorph-diff & 70.98 $\pm$ 1.52 & 0.0086 $\pm$ 0.0044 & 7.22 $\pm$ 1.03 & 69.45 $\pm$ 1.66 & 0.0364 $\pm$ 0.0157 & 7.89 $\pm$ 1.05 & 83.20 $\pm$ 1.14 & 0.0026 $\pm$ 0.0012 & \textcolor{blue}{2.06 $\pm$ 0.26} \\
        R2Net & 70.71 $\pm$ 1.49 & 0.0017 $\pm$ 0.0009 & 7.13 $\pm$ 1.01 & 68.34 $\pm$ 1.69 & 0.0086 $\pm$ 0.0031 & 8.97 $\pm$ 1.44 & 78.79 $\pm$ 1.49 & 0.0032 $\pm$ 0.0016 & 2.78 $\pm$ 0.32 \\
        NODEO & 72.73 $\pm$ 1.25 & \textcolor{blue}{0.0008 $\pm$ 0.0004} & 7.25 $\pm$ 0.97 & 70.18 $\pm$ 1.58 & 0.0023 $\pm$ 0.0011 & 7.70 $\pm$ 0.96 & 80.60 $\pm$ 1.26 & 0.0012 $\pm$ 0.0007 & 2.46 $\pm$ 0.29 \\
        NePhi & 72.39 $\pm$ 1.33 & 0.0024 $\pm$ 0.0016 & 6.75 $\pm$ 0.91 & 67.97 $\pm$ 1.91 & \textcolor{blue}{0.0019 $\pm$ 0.0005} & 8.20 $\pm$ 1.23 & 79.15 $\pm$ 1.46 & \textcolor{blue}{0.0009 $\pm$ 0.0005} & 2.36 $\pm$ 0.30 \\
        PULPo & \textcolor{blue}{74.82 $\pm$ 1.11} & 0.0009 $\pm$ 0.0005 & \textcolor{blue}{6.22 $\pm$ 0.78} & 68.69 $\pm$ 1.74 & 0.0287 $\pm$ 0.0104 & 7.92 $\pm$ 1.07 & 78.38 $\pm$ 1.57 & 0.0136 $\pm$ 0.0047 & 2.89 $\pm$ 0.33 \\
        \midrule
        TransMorph & 74.42 $\pm$ 1.33 & 0.5275 $\pm$ 0.1835 & 6.88 $\pm$ 0.85 & 71.46 $\pm$ 1.63 & 1.4883 $\pm$ 0.5328 & \textbf{7.27 $\pm$ 0.92} & 83.48 $\pm$ 1.10 & 0.1059 $\pm$ 0.0388 & 2.13 $\pm$ 0.22 \\
        TransMatch & 74.13 $\pm$ 1.31 & 0.0251 $\pm$ 0.0116 & \textcolor{red}{6.53 $\pm$ 0.83} & 71.78 $\pm$ 1.49 & 0.2483 $\pm$ 0.1020 & 7.63 $\pm$ 1.08 & 82.60 $\pm$ 1.18 & 0.0160 $\pm$ 0.0061 & 2.28 $\pm$ 0.25 \\
        DiffuseMorph & 73.11 $\pm$ 1.42 & 0.3610 $\pm$ 0.1429 & 7.21 $\pm$ 1.06 & 68.22 $\pm$ 1.77 & 1.7413 $\pm$ 0.6521 & 8.43 $\pm$ 1.30 & 77.22 $\pm$ 1.51 & 0.1277 $\pm$ 0.0412 & 2.53 $\pm$ 0.30 \\
        DiffuseReg & 72.98 $\pm$ 1.46 & 0.0723 $\pm$ 0.0289 & 7.27 $\pm$ 1.04 & 71.88 $\pm$ 1.54 & 1.0823 $\pm$ 0.4882 & 7.84 $\pm$ 1.03 & 79.31 $\pm$ 1.38 & 0.0322 $\pm$ 0.0125 & 2.31 $\pm$ 0.26 \\
        HViT & \textcolor{red}{76.57 $\pm$ 1.18} & 0.4428 $\pm$ 0.1655 & 6.74 $\pm$ 0.89 & 71.80 $\pm$ 1.72 & 1.4051 $\pm$ 0.6153 & 7.56 $\pm$ 1.11 & \textcolor{red}{83.67 $\pm$ 1.07} & 0.1368 $\pm$ 0.0435 & \textcolor{red}{2.09 $\pm$ 0.21} \\
        CorrMLP & 75.72 $\pm$ 1.27 & 0.0145 $\pm$ 0.0068 & 6.55 $\pm$ 0.87 & 71.92 $\pm$ 1.68 & 0.1895 $\pm$ 0.0782 & 7.65 $\pm$ 1.06 & 81.96 $\pm$ 1.25 & 0.0098 $\pm$ 0.0043 & 2.15 $\pm$ 0.26 \\
        SACB-Net & 73.91 $\pm$ 1.35 & \textcolor{red}{0.0099 $\pm$ 0.0041} & 6.70 $\pm$ 0.93 & \textcolor{red}{72.75 $\pm$ 1.57} & 0.3542 $\pm$ 0.1433 & \textcolor{red}{7.46 $\pm$ 0.97} & 82.75 $\pm$ 1.22 & \textcolor{red}{0.0047 $\pm$ 0.0021} & 2.19 $\pm$ 0.27 \\
        DGIR & 71.99 $\pm$ 1.44 & 0.0189 $\pm$ 0.0073 & 7.01 $\pm$ 0.89 & 68.12 $\pm$ 1.93 & \textcolor{red}{0.0674 $\pm$ 0.0213} & 7.94 $\pm$ 1.08 & 80.13 $\pm$ 1.39 & 0.0077 $\pm$ 0.0028 & 2.74 $\pm$ 0.41 \\
        \midrule
        \textbf{TPFM-DIR} & \textbf{77.87 $\pm$ 1.09}\ \ & \textbf{0.0005 $\pm$ 0.0002}\ \ & \textbf{5.85 $\pm$ 1.63}\ \ & \textbf{74.79 $\pm$ 0.98}\ \ & \textbf{0.0012 $\pm$ 0.0017}\ \ & \textcolor{blue}{7.53 $\pm$ 1.86}\ \ & \textbf{84.77 $\pm$ 1.37}\ \ & \textbf{0.0001 $\pm$ 0.0001}\ \ & \textbf{1.95 $\pm$ 0.49} \\
        \bottomrule
    \end{tabular}}
    \label{tab:brains}
\end{table}

All quantitative results are reported as mean ± standard deviation. \Cref{tab:oasis-ixi,tab:brains} present the results on the neuroimaging 3D MRI datasets. TPFM-DIR consistently achieves the highest Dice scores while maintaining near-zero folding ratios ($|J|_{<0}\%$). On average, Dice improves by 2.1\% over the next best method. Visual comparisons on OASIS are shown in \cref{fig:oasis} (top).

\Cref{tab:lungct} reports results on LungCT and AbdomenCT, where TPFM-DIR achieves zero folding on LungCT and reduces TRE by 17\% and 12\% relative to GradICON and CorrMLP, respectively. \Cref{fig:oasis} (bottom) and \cref{fig:abdomen} provide qualitative comparisons for LungCT and AbdomenCT, respectively. Finally, \Cref{tab:acdc} summarizes results on ACDC and CAMUS. On CAMUS dataset TPFM-DIR improves Dice by 3\% while preserving topology with zero folding. Visual comparisons and deformation grids are shown in \cref{fig:camus}. More visual results can be found in Appendix G.

\begin{figure}[!t]
    \centering
    \includegraphics[width=\linewidth]{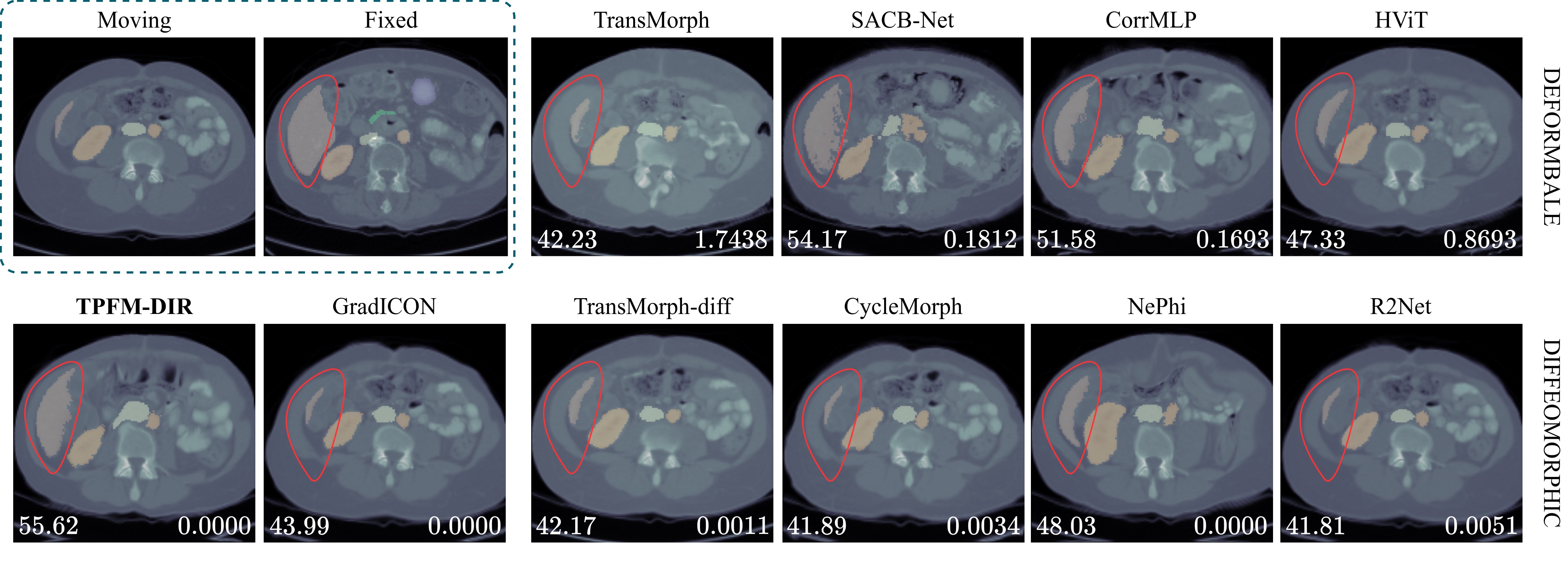}
    \caption{Qualitative comparisons on the AbdomenCT dataset. The segmentation masks are overlaid on all images. Dice score and $|J|_{<0}\%$ of each result are reported at the lower-left and lower-right corners of each sample.}
    \label{fig:abdomen}
\end{figure}

\begin{table}[!t]
    \centering
    \caption{Quantitative comparison on the LungCT and AbdomenCT datasets. The best result per metric is shown in \textbf{bold}. The second-best diffeomorphic and deformable results are highlighted in \textcolor{blue}{blue} and \textcolor{red}{red}, respectively.}
    \resizebox{\linewidth}{!}{
    \begin{tabular}{lccccccc}
        \toprule
        & \multicolumn{3}{c}{\textbf{LungCT}} & \multicolumn{4}{c}{\textbf{AbdomenCT}} \\
        \cmidrule(lr){2-4} \cmidrule(lr){5-8}
        \textbf{Method} 
        & \textbf{TRE} $\downarrow$
        & $|J|_{<0}\% \downarrow$ 
        & \textbf{SSIM} $\uparrow$ 
        & \textbf{Dice} $\uparrow$
        & $|J|_{<0}\% \downarrow$
        & \textbf{HD95} $\downarrow$
        & \textbf{SSIM} $\uparrow$ \\
        \midrule
        LDDMM & 3.09 $\pm$ 0.26 & 0.0033 $\pm$ 0.0015 & 61.14 $\pm$ 3.11 & 37.36 $\pm$ 2.71 & 0.0112 $\pm$ 0.0072 & 12.61 $\pm$ 1.46 & 53.83 $\pm$ 3.17 \\
        \midrule
        CycleMorph & 3.04 $\pm$ 0.19 & 0.7180 $\pm$ 0.2117 & 56.99 $\pm$ 3.44 & 42.56 $\pm$ 2.65 & 0.0053 $\pm$ 0.0025 & 13.04 $\pm$ 1.28 & 52.68 $\pm$ 3.02 \\
        GradICON & \textcolor{blue}{2.64 $\pm$ 0.17} & \textcolor{blue}{0.0009 $\pm$ 0.0004} & 54.06 $\pm$ 2.86 & 44.02 $\pm$ 2.81 & 0.0009 $\pm$ 0.0004 & \textcolor{blue}{12.21 $\pm$ 1.33} & 54.86 $\pm$ 3.27 \\
        TransMorph-diff & 2.89 $\pm$ 0.18 & 0.0042 $\pm$ 0.0015 & 66.85 $\pm$ 2.95 & 41.41 $\pm$ 2.72 & 0.0034 $\pm$ 0.0017 & 12.39 $\pm$ 1.25 & 62.61 $\pm$ 2.98 \\
        R2Net & 2.99 $\pm$ 0.19 & 0.0211 $\pm$ 0.0068 & 68.06 $\pm$ 2.81 & 41.99 $\pm$ 2.89 & 0.0010 $\pm$ 0.0005 & 13.05 $\pm$ 1.32 & 51.52 $\pm$ 3.36 \\
        NODEO & 2.87 $\pm$ 0.18 & 0.0038 $\pm$ 0.0013 & \textcolor{blue}{70.80 $\pm$ 2.74} & 40.37 $\pm$ 2.96 & \textcolor{blue}{0.0007 $\pm$ 0.0004} & 13.16 $\pm$ 1.38 & 52.45 $\pm$ 3.42 \\
        NePhi & 3.12 $\pm$ 0.21 & 0.0093 $\pm$ 0.0027 & 56.66 $\pm$ 3.33 & \textcolor{blue}{45.32 $\pm$ 2.24} & 0.0008 $\pm$ 0.0003 & 12.48 $\pm$ 1.19 & 66.34 $\pm$ 2.54 \\
        \midrule
        TransMorph & 2.85 $\pm$ 0.17 & 0.4767 $\pm$ 0.1546 & 68.82 $\pm$ 2.65 & 46.66 $\pm$ 2.48 & 3.1310 $\pm$ 0.8113 & 12.61 $\pm$ 1.28 & \textbf{71.84 $\pm$ 2.54} \\
        TransMatch & 3.02 $\pm$ 0.19 & 0.1343 $\pm$ 0.0482 & 67.47 $\pm$ 2.89 & 43.64 $\pm$ 2.35 & 2.6378 $\pm$ 0.6922 & 12.50 $\pm$ 1.22 & \textcolor{red}{71.47 $\pm$ 2.33} \\
        DiffuseMorph & 3.38 $\pm$ 0.21 & 0.0794 $\pm$ 0.0195 & 62.98 $\pm$ 2.92 & 39.50 $\pm$ 2.81 & 0.0292 $\pm$ 0.0113 & 12.88 $\pm$ 1.34 & 53.02 $\pm$ 3.09 \\
        DiffuseReg & 3.26 $\pm$ 0.20 & 0.7765 $\pm$ 0.2337 & 65.75 $\pm$ 2.71 & 45.91 $\pm$ 2.56 & 1.8106 $\pm$ 0.7348 & 12.55 $\pm$ 1.25 & 56.87 $\pm$ 2.88 \\
        HViT & 3.22 $\pm$ 0.20 & 0.1791 $\pm$ 0.0554 & 66.72 $\pm$ 3.41 & 47.66 $\pm$ 2.41 & 0.9601 $\pm$ 0.4175 & \textcolor{red}{10.11 $\pm$ 1.10} & 66.90 $\pm$ 2.74 \\
        CorrMLP & \textcolor{red}{2.48 $\pm$ 0.16} & 0.2673 $\pm$ 0.0711 & 65.57 $\pm$ 3.08 & 50.28 $\pm$ 2.18 & 0.1656 $\pm$ 0.0643 & 11.47 $\pm$ 1.16 & 62.66 $\pm$ 2.59 \\
        SACB-Net & 3.01 $\pm$ 0.19 & 0.9824 $\pm$ 0.2761 & \textcolor{red}{70.06 $\pm$ 2.77} & \textcolor{red}{53.38 $\pm$ 2.63} & 0.9348 $\pm$ 0.1135 & 13.09 $\pm$ 1.41 & 66.73 $\pm$ 2.96 \\
        DGIR & 3.24 $\pm$ 0.42 & \textcolor{red}{0.0711 $\pm$ 0.0428} & 67.72 $\pm$ 2.41 & 42.08 $\pm$ 2.23 & \textcolor{red}{0.0277 $\pm$ 0.0114} & 12.24 $\pm$ 1.80 & 64.48 $\pm$ 2.71 \\
        \midrule
        \textbf{TPFM-DIR} & \textbf{2.19 $\pm$ 0.14} & \textbf{0.0 $\pm$ 0.0} & \textbf{71.18 $\pm$ 2.35} & \textbf{54.54 $\pm$ 1.86} & \textbf{0.0002 $\pm$ 0.0001} & \textbf{9.81 $\pm$ 1.13} & \textcolor{blue}{68.94 $\pm$ 1.77} \\
        \bottomrule
    \end{tabular}}
    \label{tab:lungct}
\end{table}

\section{Ablation Studies}
\label{sec:ablation}

We systematically analyze key design components of TPFM-DIR. 
Specifically, we study: (i) the influence of the cocycle regularization through the weight $\lambda$ in \cref{eq:loss} on accuracy and topology preservation, (ii) the effect of progressive deformation compositions at inference, (iii) the architectural flexibility of the proposed framework, and (iv) the computational analysis of TPFM-DIR.

\subsection{Effect of Regularization Factor}
\label{sec:abl-lambda}
The regularization weight $\lambda$ controls the trade-off between alignment accuracy and topology preservation. The value $\lambda=10$ is fixed across all datasets and used in all main experiments since it consistently provided a favorable trade-off between alignment accuracy and topology preservation across datasets. We evaluate the impact of varying $\lambda$ on three representative datasets. \Cref{tab:reg-factor} shows that increasing $\lambda$ progressively eliminates voxels with negative Jacobian determinant, at the cost of a moderate decrease in the performance metric. Conversely, reducing $\lambda$ does not improve alignment accuracy, while significantly increasing folding. In the extreme case of $\lambda=0$, both the performance metric and topology deteriorate substantially. This confirms that cocycle regularization is essential for stabilizing the learned deformation field and promoting physically plausible solutions of the underlying non-autonomous ODE.

\begin{table}[!t]
    \centering
    \caption{Quantitative comparison on the ACDC and CAMUS datasets. The best result per metric is shown in \textbf{bold}. The second-best result is highlighted in \textcolor{red}{red}.}
    \resizebox{\linewidth}{!}{
    \begin{tabular}{lcccccccccc}
        \toprule
        & \multicolumn{5}{c}{\textbf{ACDC}} & \multicolumn{5}{c}{\textbf{CAMUS}} \\
        \cmidrule(lr){2-6} \cmidrule(lr){7-11}
        \textbf{Method} 
        & \textbf{Dice $\uparrow$} 
        & $|J|_{<0}\% \downarrow$ 
        & \textbf{HD95} $\downarrow$
        & \textbf{ASSD} $\downarrow$
        & \textbf{SDLogJ} $\downarrow$
        & \textbf{Dice} $\uparrow$
        & $|J|_{<0}\% \downarrow$ 
        & \textbf{HD95} $\downarrow$
        & \textbf{ASSD} $\downarrow$
        & \textbf{SDLogJ} $\downarrow$ \\
        \midrule
        LDDMM & 78.17 $\pm$ 1.52 & 0.0062 $\pm$ 0.0044 & \textcolor{red}{5.43 $\pm$ 0.63} & \textcolor{red}{2.02 $\pm$ 0.96} & 0.197 $\pm$ 0.131 & 79.30 $\pm$ 5.21 & 0.0019 $\pm$ 0.0013 & 6.25 $\pm$ 1.35 & 2.61 $\pm$ 0.92 & 0.195 $\pm$ 0.031 \\
        \midrule
        AdaCS & \textcolor{red}{85.46 $\pm$ 1.12} & 0.1493 $\pm$ 0.0185 & 5.56 $\pm$ 0.34 & 2.12 $\pm$ 1.20 & 0.326 $\pm$ 0.131 & \textcolor{red}{82.67 $\pm$ 5.14} & 0.0029 $\pm$ 0.0130 & \textcolor{red}{5.23 $\pm$ 1.95} & \textcolor{red}{2.11 $\pm$ 0.76} & 0.193 $\pm$ 0.033 \\
        TransMorph-diff & 82.13 $\pm$ 1.42 & \textcolor{red}{0.0012 $\pm$ 0.0005} & 6.55 $\pm$ 0.43 & 2.67 $\pm$ 1.36 & \textbf{0.162 $\pm$ 0.069} & 78.63 $\pm$ 5.54 & \textcolor{red}{0.0011 $\pm$ 0.0016} & 10.15 $\pm$ 3.16 & 4.04 $\pm$ 1.12 & \textcolor{red}{0.166 $\pm$ 0.028} \\
        TransMorph & 84.47 $\pm$ 1.18 & 0.0418 $\pm$ 0.0087 & 6.24 $\pm$ 0.36 & 2.54 $\pm$ 1.22 & 0.575 $\pm$ 0.094 & 80.53 $\pm$ 5.93 & 0.1017 $\pm$ 0.1376 & 9.10 $\pm$ 3.16 & 3.68 $\pm$ 1.21 & 0.312 $\pm$ 0.075 \\
        DiffuseMorph & 81.90 $\pm$ 1.35 & 0.0383 $\pm$ 0.0091 & 7.34 $\pm$ 0.47 & 2.81 $\pm$ 1.48 & 0.251 $\pm$ 0.092 & 76.43 $\pm$ 6.72 & 0.0932 $\pm$ 0.1329 & 9.77 $\pm$ 3.18 & 4.28 $\pm$ 1.27 & 0.393 $\pm$ 0.087 \\
        DiffuseReg & 83.31 $\pm$ 1.21 & 0.2345 $\pm$ 0.0228 & 6.17 $\pm$ 0.39 & 2.35 $\pm$ 1.31 & 0.398 $\pm$ 0.107 & 77.29 $\pm$ 6.19 & 0.4026 $\pm$ 0.1402 & 10.60 $\pm$ 3.04 & 4.22 $\pm$ 1.23 & 0.395 $\pm$ 0.062 \\
        DGIR & 82.39 $\pm$ 2.97 & 0.0059 $\pm$ 0.0125 & 7.43 $\pm$ 1.08 & 2.26 $\pm$ 0.79 & 0.241 $\pm$ 0.021 & 81.83 $\pm$ 6.27 & 0.0315 $\pm$ 0.0282 & 6.71 $\pm$ 2.02 & 2.78 $\pm$ 0.82 & 0.301 $\pm$ 0.066 \\
        \midrule
        \textbf{TPFM-DIR} & \textbf{87.42 $\pm$ 0.99} & \textbf{0.0002 $\pm$ 0.0001} & \textbf{5.36 $\pm$ 0.35} & \textbf{1.91 $\pm$ 1.07} & \textcolor{red}{0.182 $\pm$ 0.023} & \textbf{85.12 $\pm$ 2.12} & \textbf{0.0 $\pm$ 0.0} & \textbf{4.72 $\pm$ 2.19} & \textbf{2.06 $\pm$ 0.94} & \textbf{0.162 $\pm$ 0.020} \\
        \bottomrule
    \end{tabular}}
    \label{tab:acdc}
\end{table}

\begin{table}[!t]
    \centering
    \caption{Effect of the regularization weight $\lambda$ on performance and topology preservation of TPFM-DIR across OASIS, LungCT, and ACDC.}
    \resizebox{0.5\linewidth}{!}{
    \begin{tabular}{lcccccc}
        \toprule
        & \multicolumn{2}{c}{\textbf{OASIS}} & \multicolumn{2}{c}{\textbf{LungCT}} & \multicolumn{2}{c}{\textbf{ACDC}} \\
        \cmidrule(lr){2-3} \cmidrule(lr){4-5} \cmidrule(lr){6-7}
        $\lambda$
        & \textbf{Dice} $\uparrow$ 
        & $|J|_{<0}\% \downarrow$
        & \textbf{TRE} $\downarrow$ 
        & $|J|_{<0}\% \downarrow$
        & \textbf{Dice} $\uparrow$ 
        & $|J|_{<0}\% \downarrow$\\
        \midrule
        0 & 79.48 & 3.1263 & 4.02 & 2.1821 & 80.19 & 3.0671 \\
        5 & 86.90 & 0.0415 & 3.36 & 0.0019 & 84.52 & 0.0039 \\
        10 & \textbf{87.08} & 0.0019 & \textbf{2.19} & \textbf{0.0} & \textbf{87.42} & 0.0002 \\
        15 & 85.84 & 0.0001 & 2.41 & \textbf{0.0} & 87.08 & 0.0001 \\
        20 & 85.02 & \textbf{0.0} & 3.11 & \textbf{0.0} & 85.74 & \textbf{0.0} \\
        \bottomrule
    \end{tabular}}
    \label{tab:reg-factor}
\end{table}

\subsection{Effect of Inference Compositions}
\label{sec:abl-integrate}

Progressive compositions in TPFM-DIR are performed only at the inference. We observe that as few as four compositions are sufficient to obtain near-zero $|J|_{<0}\%$. \Cref{fig:num-steps} illustrates the effect of increasing the number of steps on Dice and $|J|_{<0}\%$. Although a single-step prediction is feasible, progressive compositions improve smoothness and reduce folding, with only a minor degradation in alignment performance. Beyond four steps, improvements in topology preservation saturate, making this choice a practical trade-off between computational cost and deformation regularity.

\begin{figure}[!ht]
    \centering
    \includegraphics[width=0.9\linewidth]{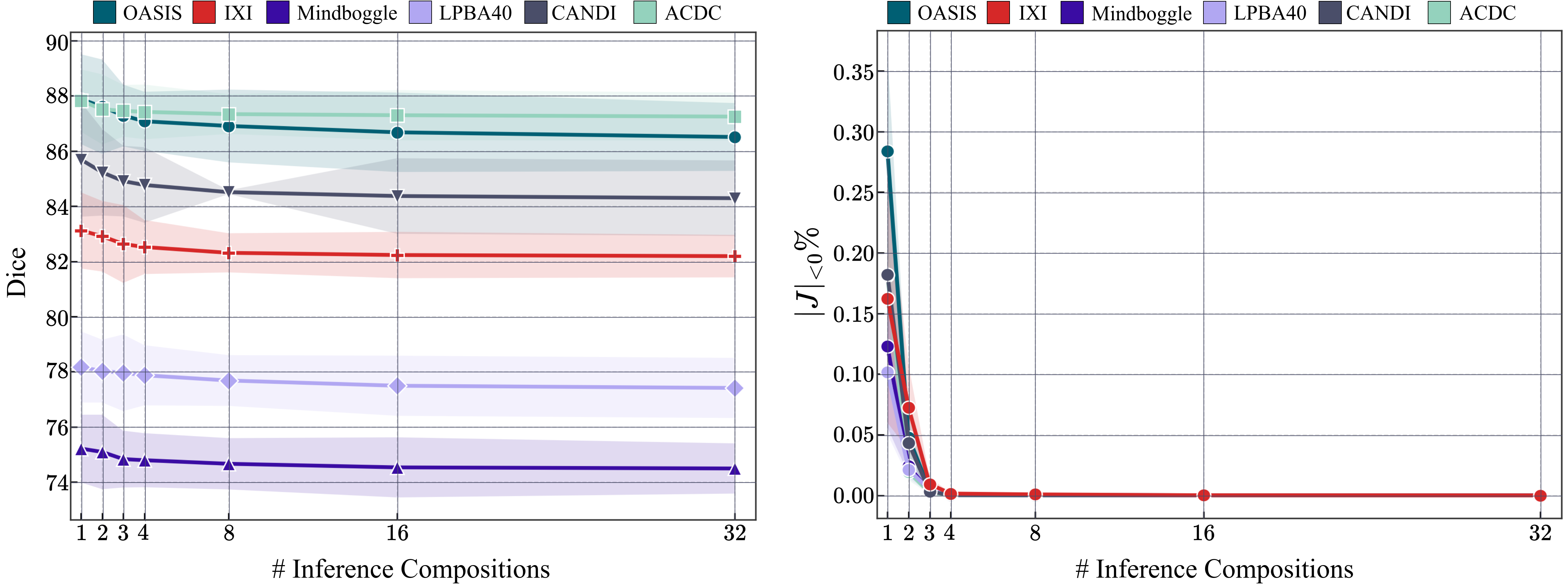}
    \caption{The effect of number of compositions at inference on the Dice score and $|J|_{<0}\%$.}
    \label{fig:num-steps}
\end{figure}


\subsection{The Choice of Architecture}
\label{sec:abl-arch}

\begin{figure}[!t]
    \centering

    \includegraphics[width=\linewidth]{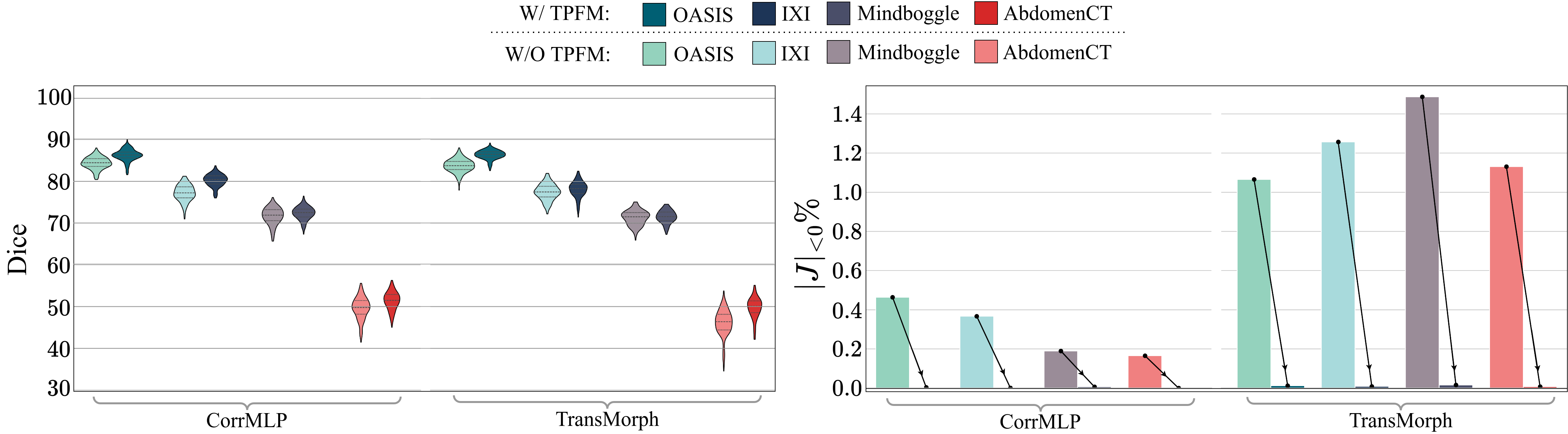}

    \vspace{1em}
    
    \resizebox{\linewidth}{!}{
    \begin{tabular}{lcccccccc}
        \toprule
        & \multicolumn{2}{c}{\textbf{OASIS}} 
        & \multicolumn{2}{c}{\textbf{IXI}} 
        & \multicolumn{2}{c}{\textbf{Mindboggle}} 
        & \multicolumn{2}{c}{\textbf{AbdomenCT}} \\
        
        \cmidrule(lr){2-3}
        \cmidrule(lr){4-5}
        \cmidrule(lr){6-7}
        \cmidrule(lr){8-9}
        
        \textbf{Method} 
        & \textbf{Dice} $\uparrow$ & $|J|_{<0}\% \downarrow$ 
        & \textbf{Dice} $\uparrow$ & $|J|_{<0}\% \downarrow$ 
        & \textbf{Dice} $\uparrow$ & $|J|_{<0}\% \downarrow$ 
        & \textbf{Dice} $\uparrow$ & $|J|_{<0}\% \downarrow$ \\
        
        \midrule
        CorrMLP & 84.66 $\pm$ 1.24 & 0.4640 $\pm$ 0.1776 & 77.54 $\pm$ 1.68 & 0.3675 $\pm$ 0.2168  & 71.92 $\pm$ 1.68 & 0.1895 $\pm$ 0.0782 & 50.28 $\pm$ 2.18 & 0.1656 $\pm$ 0.0643 \\
        CorrMLP-diff & 83.12 $\pm$ 1.73 & 0.0377 $\pm$ 0.0112 & 75.14 $\pm$ 2.08 & 0.0782 $\pm$ 0.0511 & 71.54 $\pm$ 1.88 & 0.0653 $\pm$ 0.0233 & 48.73 $\pm$ 2.98 & 0.0114 $\pm$ 0.0103 \\
        \midrule
        CorrMLP-TPFM & \textbf{86.32 $\pm$ 1.10} & 0.0038 $\pm$ 0.0015 & \textbf{80.38 $\pm$ 1.21} & 0.0016 $\pm$ 0.0012 & \textbf{72.77 $\pm$ 1.33} & 0.0067 $\pm$ 0.0029 & \textbf{51.47 $\pm$ 1.88} & 0.0002 $\pm$ 0.0003 \\
        \midrule
        \midrule
        TransMorph & 84.11 $\pm$ 1.30 & 1.0665 $\pm$ 0.5631 & 77.72 $\pm$ 1.76 & 1.2574 $\pm$ 0.7574 & 71.46 $\pm$ 1.63 & 1.4883 $\pm$ 0.5328 & 46.66 $\pm$ 2.48 & 3.1310 $\pm$ 0.8113 \\
        TransMorph-diff & 83.51 $\pm$ 1.52 & 0.0066 $\pm$ 0.0073 & 75.98 $\pm$ 2.01 & 0.0018 $\pm$ 0.0022 & 69.45 $\pm$ 1.66 & 0.0364 $\pm$ 0.0157 & 41.41 $\pm$ 2.72 & 0.0034 $\pm$ 0.0017 \\
        \midrule
        TransMorph-TPFM & \textbf{86.50 $\pm$ 0.98} & 0.0122 $\pm$ 0.0071 & \textbf{78.19 $\pm$ 1.67} & 0.0092 $\pm$ 0.0043 & \textbf{71.96 $\pm$ 1.45} & 0.0161 $\pm$ 0.0098 & \textbf{50.27 $\pm$ 2.11} & 0.0083 $\pm$ 0.0022 \\
        \bottomrule
    \end{tabular}}    
    \caption{
    Quantitative (table) and visual (figure) analysis of improvements obtained by training CorrMLP and TransMorph in the TPFM-DIR setting. Dice scores increase while $|J|_{<0}\%$ is significantly reduced across OASIS, IXI, Mindboggle, and AbdomenCT.
    }
    \label{fig:before-after-combined}
\end{figure}

TPFM-DIR is a diffeomorphic registration framework rather than a specific network architecture. Although the backbone must support temporal conditioning (e.g., via context addition or FiLM-based modulation \cite{film}), the framework is not restricted to a particular design. To validate architectural flexibility, we adapt two strong deformable baselines, CorrMLP and TransMorph, by modifying only their decoders to incorporate time context via context addition (see Appendix E for more details). \Cref{fig:before-after-combined} reports the results.

For each backbone, we compare three settings: the original deformable model, a conventional diffeomorphic variant trained with scaling and squaring (\textit{-diff}), and the proposed TPFM-DIR formulation. In both architectures, TPFM-DIR improves Dice while drastically reducing folding percentage relative to the original deformable models. Moreover, compared with the scaling-and-squaring variants, TPFM-DIR consistently achieves higher Dice scores while maintaining comparable or better topology preservation. These findings demonstrate that TPFM-DIR is architecture-agnostic and can enhance topology preservation without sacrificing alignment accuracy.

\subsection{Computational Analysis}
\label{sec:abl-comp}

To assess the computational characteristics of TPFM-DIR, we report the number of parameters, training and inference runtime per image pair, and GPU memory consumption in \Cref{tab:comp}. We compare against a strong deformable model (HViT), a scaling-and-squaring-based diffeomorphic method (TransMorph-diff), a learning-based non-autonomous approach (R2Net), and an instance-optimization non-autonomous method (NODEO). Since NODEO is a pair-wise optimization method, its training statistics are reported with ``-''. The performance is measured on OASIS as a representative large 3D dataset. TPFM-DIR achieves the fastest inference time while maintaining a low inference memory footprint. Comparing to models incorporating non-autonomous ODEs, TPFM-DIR exhibits faster training and inference times due to avoiding numerical integration during optimization. All measurements are obtained on an NVIDIA RTX3090 GPU.

\section{Conclusion}
\label{sec:conclusion}

\begin{table}[!t]
    \centering
    \caption{Computational analysis across methods for per-pair training/inference iterations. ``Auto.'' and ``Non-Auto.'' denote autonomous and non-autonomous ODE settings.}
    \resizebox{0.9\linewidth}{!}{
    \begin{tabular}{lcccccc}
        \toprule
        \textbf{Method} & \textbf{\# Params (M)} & \textbf{Train Time (s)} & \textbf{Inference Time (s)} & \textbf{Train Mem. (GB)} & \textbf{Inference Mem. (GB)} & \textbf{ODE Type} \\
        \midrule
        TransMorph-diff & 46.8 & \textbf{0.98} & 0.47 & 12.3 & 4.3 & Auto. \\
        NODEO & \textbf{4.2} & - & 214 & - & 3.5 & Non-Auto. \\
        R2Net & 11.22 & 1.53 & 0.96 & \textbf{11.8} & 3.3 & Non-Auto. \\
        HViT & 21.2 & 1.57 & 0.41 & 24.6 & 5.76 & - \\
        \midrule
        TPFM-DIR & 30.6 & 1.31 & \textbf{0.38} & 18.1 & \textbf{2.7} & Non-Auto. \\
        \bottomrule
    \end{tabular}}
    \label{tab:comp}
\end{table}

We introduced TPFM-DIR, a diffeomorphic image registration framework that directly parameterizes the flow map of a non-autonomous differential equation. By enforcing the anchored cocycle property during training, the method learns temporally consistent diffeomorphic trajectories without requiring numerical integration during optimization. Extensive experiments across nine datasets spanning multiple modalities, anatomical regions, and spatial dimensions demonstrate that TPFM-DIR substantially narrows the long-standing performance gap between diffeomorphic models and strong unconstrained deformable methods. In many settings, it achieves superior registration accuracy while preserving topology. Importantly, TPFM-DIR is architecture-agnostic and can be integrated into existing backbones with minimal modifications, resulting in consistent improvements of deformation regularity and empirical performance. Beyond medical image registration, the proposed formulation provides a principled framework for learning solution operators of non-autonomous ODEs in a data-driven manner.

\clearpage  

\section*{Acknowledgements}
The authors acknowledge funding support from NSERC Discovery Grants and the Department of Computing Science, University of Alberta. The project received compute support from Digital Research Alliance of Canada.

%
%
\bibliographystyle{splncs04}
\bibliography{main}

\clearpage
\appendix
\renewcommand{\theHsection}{\Alph{section}}
\section{Proofs}
\label{app:proof}
The proof sketch of \cref{thm} is as follows: for a sufficiently small $t-s$, we show that $\phi_{s,t}$ in \cref{eq:phi_model} is a local diffeomorphism. Then we construct $\phi_{0,t}$ with progressively applying \cref{eq:partial-coc}, deducing that $\phi_{0,t}$ is a diffeomorphism. Finally, we derive the full cocycle property using \cref{eq:partial-coc} and invertibility of $\phi_{0,t}$, proving that $\phi_{s,t}$ is a two-parameter diffeomorphism solving a non-autonomous ODE of \cref{eq:naode}.

Since our implementation uses all Lipschitz-continuous layers (including convolution and linear layers and SiLU activations), we assume that the Lipschitz constant of the spatial gradient of network $f_\theta$ is bounded by $\sup\|Df_\theta\|_2 \leq L_f$ and also $f_\theta$ is at least $C^1$-differentiable. With this established we provide the proof for \cref{thm}.

\bigskip
\noindent\textbf{\cref{thm}.}
\emph{
Under standard regularity assumptions, a deformation $\phi_{s,t}$ satisfying the identity condition and the anchored cocycle properties of \cref{eq:partial-coc} is a two-parameter diffeomorphism solving a non-autonomous ODE of \cref{eq:naode} on a bounded open domain $\Omega$.
}
\begin{proof}
    By assumption above $f_\theta$ is $C^1$ in $x$ and satisfies 
    \[
    \sup_{x\in\Omega}\|D_x f_\theta(s,t,\cdot)\| \le L_f.
    \]

    where $D_x$ is the spatial derivative operator.

    \paragraph{Step 1: Local diffeomorphism for small time increments.}

    Let $0=t_0 < t_1 < \cdots < t_m = t$ be a partition such that 
    $\delta := t_{i+1}-t_i < 1/L_f$. 
    
    From \cref{eq:phi_model},
    \[
    \phi_{t_i,t_{i+1}}(x)
    = x + (t_{i+1}-t_i) f_\theta(t_i,t_{i+1},x)
    = x + u_i(x),
    \]
    where $u_i(x) := \delta f_\theta(t_i,t_{i+1},x)$. Then
    \[
    D_x \phi_{t_i,t_{i+1}}
    = I + D_x u_i,
    \quad
    \|D_x u_i\| \le \delta L_f < 1.
    \]
    
    Hence $\|D_x\phi - I\|<1$. By the Neumann series argument \cite{neumann}, $D_x\phi_{t_i,t_{i+1}}$ is invertible for all $x$, and therefore, by the Inverse Function Theorem, each $\phi_{t_i,t_{i+1}}$ is a local $C^1$-diffeomorphism.

    \paragraph{Step 2: Global diffeomorphism of $\phi_{0,t}$.}

    From the anchored cocycle condition \cref{eq:partial-coc},
    \[
    \phi_{0,t}
    = \phi_{t_{m-1},t_m}
    \circ \cdots \circ
    \phi_{t_0,t_1}.
    \]
    Since compositions of local diffeomorphisms are local diffeomorphisms and each map is invertible for sufficiently small increments, $\phi_{0,t}$ is a global $C^1$-diffeomorph\-ism.
    
    Moreover, from \cref{eq:partial-coc} with $t=0$,
    \[
    \phi_{s,0} = \phi_{0,s}^{-1}.
    \]

    \paragraph{Step 3: Full cocycle property.}

    Using the first anchored condition,
    \[
    \phi_{s,t}\circ \phi_{0,s} = \phi_{0,t}.
    \]
    Right-multiplying by $\phi_{s,0} = \phi_{0,s}^{-1}$ yields
    \[
    \phi_{s,t}
    = \phi_{0,t} \circ \phi_{s,0}.
    \]
    Therefore, for any $r\leq s\leq t$,
    \[
    \phi_{s,t}\circ \phi_{r,s}
    = (\phi_{0,t}\circ \phi_{s,0})
    \circ
    (\phi_{0,s}\circ \phi_{r,0})
    = \phi_{0,t}\circ \phi_{r,0}
    = \phi_{r,t},
    \]
    which establishes the full cocycle property. Since $\phi_{s,t}$ is a $C^1$ mapping satisfying the identity and cocycle conditions, $\phi_{s,t}$ is a two-parameter flow map solving a non-autonomous ODE \cite{nonauto}.
\end{proof}


\bigskip
\noindent\textbf{\cref{corr}.}
\emph{
For the flow map $\phi_{s,t}$ parameterized as \cref{eq:phi_model}, the instantaneous velocity can be analytically obtained via
    \[
    v_t = f_\theta(t, t; I_m, I_f).
    \]
}
\begin{proof}
    From \cref{eq:phi_model},
    \[
    \phi_{s,t}(x)
    =
    x + (t-s) f_\theta(s,t,x).
    \]
    Differentiating with respect to $t$ yields
    \begin{equation}
        \label{eq:partial-vel}
        \partial_t \phi_{s,t}(x)
        =
        f_\theta(s,t,x)
        +
        (t-s)\,\partial_t f_\theta(s,t,x).
    \end{equation}
    
    \smallskip
    \noindent
    The instantaneous velocity field of the flow is defined as
    \[
    v(x,t)
    =
    \left.\partial_t \phi_{s,t}(x)\right|_{s=t}.
    \]
    Evaluating \cref{eq:partial-vel} at $s=t$ eliminates the second term, giving
    \[
    v(x,t)
    =
    f_\theta(t,t,x).
    \]
\end{proof}

\section{Discussion on Inference Compositions}
\label{app:compos}

\begin{figure}[!t]
    \centering
    \includegraphics[width=0.8\linewidth]{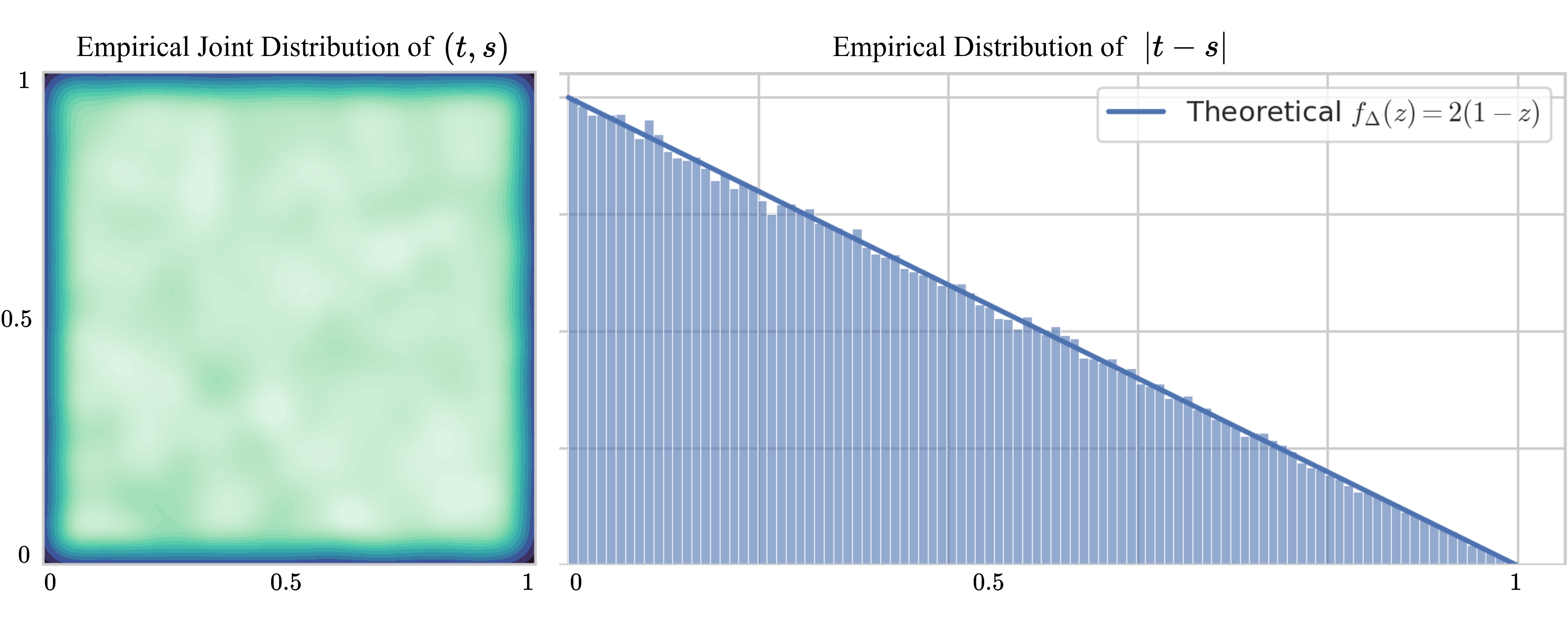}
    \caption{The comparison between the joint distribution of time samples $s,t\sim Uni(0,1)$ and the distribution of time differences $|t-s|$}
    \label{fig:empirical}
\end{figure}

The use of composition at inference is motivated by the bias induced by drawing $s,t \sim Uni(0,1)$. Let $\Delta = |t-s|$. We compute the cumulative distribution function
\[
F_{\Delta}(z)
=
P(|t-s| \le z),
\quad 0 \le z \le 1.
\]
Since $(s,t)$ is uniformly distributed over the unit square, the event $|t-s| \le z$ corresponds to the band between the lines $t = s+z$ and $t = s-z$. A direct geometric calculation yields
\[
F_{\Delta}(z)
=
2z - z^2,
\]
and differentiation gives the probability density function
\[
f_{\Delta}(z)
=
\begin{cases}
2(1 - z), & 0 \le z \le 1, \\
0, & \text{otherwise}.
\end{cases}
\]

The density decreases linearly with $z$, implying that smaller time increments are sampled more frequently than larger ones (\cref{fig:empirical}). Although the network observes the full range $[0,1]$, larger transitions are relatively less represented during optimization. As a result, directly predicting $\phi_{0,1}$ in a single step may be less consistent with the flow structure than composing multiple locally predicted maps.

By partitioning $[0,1]$ into $N$ subintervals and composing the corresponding deformations, inference primarily relies on well-trained small increments. This improves stability and promotes better preservation of diffeomorphic structure. In practice, we find that $N=4$ provides a favorable balance between computational efficiency and deformation accuracy.

\section{Datasets and Preprocessing}
\label{app:data}

\begin{table}[!ht]
    \centering
    \caption{Dataset statistics and data splits used in experiments. For each dataset, the sum of training, validation, and test scans equals the total number of available scans.}
    \resizebox{.8\linewidth}{!}{
    \begin{tabular}{lllcccccc}
        \toprule
        \multirow[b]{2}{*}{\textbf{Dataset}} 
        & \multirow[b]{2}{*}{\textbf{Modality}} 
        & \multirow[b]{2}{*}{\textbf{Organ}}
        & \multicolumn{2}{c}{\textbf{Train}}
        & \multicolumn{2}{c}{\textbf{Val}}
        & \multicolumn{2}{c}{\textbf{Test}} \\
        \cmidrule(lr){4-5} \cmidrule(lr){6-7} \cmidrule(lr){8-9}
        & & 
        & \# Scans & \# Pairs
        & \# Scans & \# Pairs
        & \# Scans & \# Pairs \\
        \midrule
        OASIS & MRI & Brain & 256 & 1000 & 60 & 100 & 100 & 1000 \\
        IXI & MRI & Brain & 450 & 1500 & 50 & 100 & 81 & 300 \\
        Mindboggle101 & MRI & Brain & 70 & 1000 & 10 & 45 & 21 & 210 \\
        LPBA40 & MRI & Brain & 20 & 190 & 5 & 10 & 15 & 105 \\
        CANDI & MRI & Brain & 80 & 400 & 11 & 25 & 26 & 325 \\
        ACDC & MRI & Heart & 90 & 90 & 15 & 15 & 45 & 45 \\
        CAMUS & Ultrasound & Heart & 400 & 400 & 50 & 50 & 50 & 50 \\
        LungCT & CT & Lung & 20 & 20 & 3 & 3 & 9 & 9 \\
        AbdomenCT & CT & Abdomen & 25 & 300 & 5 & 10 & 20 & 190 \\
        \bottomrule
    \end{tabular}}
    \label{tab:datasets}
\end{table}

\begin{enumerate}
    \item \textbf{OASIS \cite{oasis}} The Open Access Series of Imaging Studies (OASIS)\footnote{\url{https://github.com/adalca/medical-datasets/blob/master/neurite-oasis.md}} contains 416 T1-weighted brain MR scans of subjects aged 18–96, including 100 diagnosed with mild to moderate Alzheimer's disease. The segmentation masks of 35 subcortical regions available in the dataset are used for evaluation.
    \item \textbf{IXI\footnote{\url{https://brain-development.org/ixi-dataset}}} The Information eXtraction from Images (IXI) dataset contains 581 brain scans. We use FreeSurfer to obtain skull-stripped volumes and segmentation masks comprising 32 subcortical regions for evaluation.
    \item \textbf{LPBA40 \cite{lpba}} The LONI Probabilistic Brain Atlas (LPBA40)\footnote{\url{https://www.loni.usc.edu/research/atlas_downloads}} contains the brain scans of 40 subjects and segmentation masks of 56 cortical and subcortical regions used for evaluations.
    \item \textbf{Mindboggle101\cite{mindboggle}} The Mindboggle101\footnote{\url{https://mindboggle.info}} contains T1 weighted brain MRI scans of 101 healthy subjects with the segmentation masks of cortical regions used for evaluations.
    \item \textbf{CANDI\cite{candi}} The Child and Adolescent NeuroDevelopment Initiative (CAN\-DI) \footnote{\url{https://www.nitrc.org/projects/candi_share/}} dataset contains T1 weighted brain scans of 117 subjects divided into four distinct subgroups of Healthy Control, Schizophrenia Spectrum, Bipolar Disorder with Psychosis, and Bipolar Disorder without Psychosis. The segmentation masks of 32 subcortical regions are available in the dataset and used for evaluations.
    \item \textbf{LungCT\footnote{\label{l2r}\url{https://learn2reg.grand-challenge.org/Datasets/}}} The LungCT dataset, obtained from the Learn2Reg Challenge, contains thoracic CT scans acquired at different respiratory phases. The dataset comes with 300 keypoints per image annotated by experts. The keypoints are used for evaluation and computing the target registration error (TRE) metric.
    \item \textbf{AbdomenCT\footref{l2r}} The AbdomenCT dataset, also from the Learn2Reg Challenge, contains abdominal CT scans of 50 subjects with multi-organ annotations for evaluation.
    \item \textbf{ACDC \cite{acdc}} The Automated Cardiac Diagnosis Challenge (ACDC)\footnote{\url{https://www.creatis.insa-lyon.fr/Challenge/acdc/}} dataset is a 2D cardiac MRI dataset containing short-axis cine MR scans of 150 patients.
    \item \textbf{CAMUS \cite{camus}} The Cardiac Acquisitions for Multi-structured Ultrasound Segmentation (CAMUS)\footnote{\url{https://www.creatis.insa-lyon.fr/Challenge/camus/}} is a 2D cardiac ultrasound dataset comprising 500 patients. The segmentation masks of 2 chamber and 4 chamber structures are available in the dataset. We use the 2-chamber (2CH) segmentation masks for evaluations in accordance with \cite{adacs}.
\end{enumerate}

\Cref{tab:datasets} summarizes the datasets used in our evaluations, the number of scans and pairs used in training, validation, and test. For 3D brain MRI datasets and AbdomenCT, moving–fixed pairs are randomly sampled from distinct subjects within each split, without cross-split overlap. Pair sampling is performed independently for training, validation, and test sets according to the numbers reported in \Cref{tab:datasets}. For LungCT dataset, registration is performed between the end-inhalation and end-exhalation scans of the same patient. For ACDC MRI and CAMUS Ultrasound datasets, registration is performed between the end-systolic (ES) and end-diastolic (ED) phases for each patient. For the ACDC dataset we randomly chose the patients for training, validation, and test sets. For CAMUS, we adopt the official patient-level splits provided with the dataset.

\textbf{Preprocessing.} We make sure all brain scans are in MNI152 space and center-cropped and resampled to a common spatial resolution of $160\times144\times192$. For LungCT and AbdomenCT, intensities are clipped to the Hounsfield Unit range $[-1000, 1000]$ to focus on tissue contrast. ACDC scans are center-cropped to $128\times128$ while CAMUS scans are \textit{resized} to $128\times128$. As a final step, all image intensities are min-max normalized to $[0, 1]$.

\section{Implementation Details}
\label{app:impl}
We provide the implementation details of TPFM-DIR. As shown in \cref{fig:arch} (left panel), TPFM-DIR adopts a time-embedded UNet architecture. The encoder layers dimensions are $[2, 16, 32, 64, 128, 512]$, where the input channel size 2 corresponds to the concatenated moving and fixed images.  The decoder layers dimensions are $[512, 128, 64, 32, 16, 3]$. For the 2D setting the last dimension of the decoder is $2$ accounting for 2D deformations. We use skip connections between corresponding encoder and decoder stages, equipped with gated attention modules. The skip features are concatenated with the decoder feature maps. We use SiLU\cite{silu} for all nonlinearities in the network due to its smoothness properties.

To incorporate the time context, each time sample $s$ and $t$ is initially encoded via 256 dimensional sinusoidal positional encoding followed by separate linear layers. The final embeddings of $s$ and $t$ are then summed and passed to each decoder layer. The time context is incorporated into each decoder layer via FiLM modulation \cite{film}. More precisely, if $h_i$ is the feature map of the $i$-th decoder layer, the time context $c$ is mixed with $h_i$ through:

\[
\hat{h}_i = (1 + \gamma(c)) h_i + \zeta(c),
\]

where $\gamma(.)$ and $\zeta(.)$ are learnable linear layers. Other methods are directly taken from their official repositories. The implementation of the LDDMM was taken from the widely used PyTorch package \footnote{\url{https://github.com/brianlee324/torch-lddmm}}.

\section{TPFM-DIR Backbone Architectures}
\label{app:backbones}

\begin{figure}
    \centering
    \includegraphics[width=\linewidth]{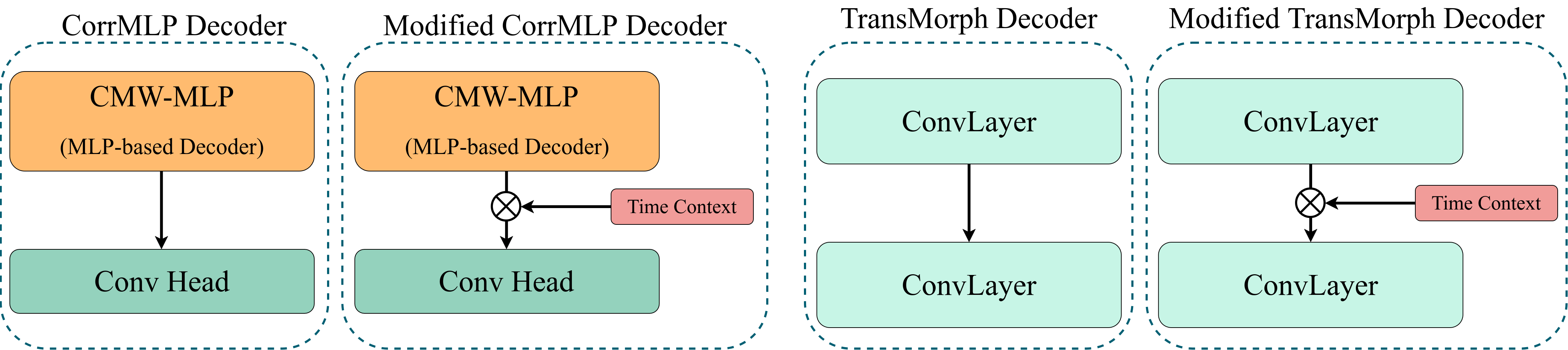}
    \caption{Illustration of the modifications applied to the CorrMLP and TransMorph decoders to adapt them to the TPFM-DIR setting. Time context is injected into the decoder layers via FiLM modulation \cite{film}.}
    \label{fig:backbones}
\end{figure}

\Cref{sec:abl-arch} presents the results of incorporating different backbones into the TPFM-DIR setting. We adapt the CorrMLP \cite{corrmlp} and TransMorph\cite{transmorph} decoders to incorporate time dependence and evaluate the resulting models within the TPFM-DIR framework. We detail the architectural modifications applied to the CorrMLP and TransMorph decoders to adapt them to the TPFM-DIR setting. \Cref{fig:backbones} shows how the time context is injected into the decoder layers to make the decoders time-dependent and suitable for TPFM-DIR setting. The time context is computed as described in Appendix~\ref{app:impl} and injected into the decoder layers via FiLM modulation. We do not modify any other architectural components of the original models, including encoder and decoder channel dimensions and activation functions, ensuring that the changes remain minimal. The additional FiLM layers introduce only lightweight linear projections for time conditioning and do not alter the depth or structure of the original decoders. The reported results in \cref{fig:before-after-combined} are obtained by training the modified CorrMLP and TransMorph from scratch within the TPFM-DIR setting.

\section{Complete Quantitative Results}
\label{app:full-res}

\Cref{tab:oasis_ixi_appendix} presents the full-metric results for OASIS, IXI, and LPBA40 datasets. \Cref{tab:candi_mindboggle_appendix} presents the full-metric results for CANDI, Mindboggle101, and AbdomenCT datasets. \Cref{tab:lungct_only} reports the results on LungCT dataset with full set of metrics, and \cref{tab:acdc_camus_appendix} shows the results with complete metrics over the 2D cardiac datasets.

\begin{table*}[!t]
    \centering
    \caption{Quantitative comparison on the OASIS, IXI, and LPBA40 datasets. 
    The best result per metric is shown in \textbf{bold}. 
    The second-best diffeomorphic and deformable results are highlighted in 
    \textcolor{blue}{blue} and \textcolor{red}{red}, respectively.}
    \resizebox{0.85\linewidth}{!}{
    \begin{tabular}{lcccccc}
        \toprule
        \multicolumn{7}{c}{\textbf{OASIS}} \\
        \midrule
        \textbf{Method} 
        & \textbf{Dice $\uparrow$} 
        & $|J|_{<0}\% \downarrow$ 
        & \textbf{HD95} $\downarrow$ 
        & \textbf{ASSD} $\downarrow$ 
        & \textbf{SSIM} $\uparrow$
        & \textbf{SDLogJ} $\downarrow$ \\
        \midrule
        LDDMM & 76.59 $\pm$ 2.42 & 0.0064 $\pm$ 0.0051 & 3.89 $\pm$ 0.93 & 1.18 $\pm$ 0.36 & 80.24 $\pm$ 1.84 & 0.012 $\pm$ 0.009 \\
        \midrule
        CycleMorph & 81.93 $\pm$ 2.14 & 0.0211 $\pm$ 0.0091 & 2.36 $\pm$ 0.81 & \textcolor{blue}{0.75 $\pm$ 0.26} & 88.31 $\pm$ 1.43 & 0.040 $\pm$ 0.026 \\
        GradICON & \textcolor{blue}{83.74 $\pm$ 1.42} & 0.0039 $\pm$ 0.0012 & \textcolor{blue}{2.09 $\pm$ 0.36} & 0.78 $\pm$ 0.36 & 90.17 $\pm$ 1.19 & 0.011 $\pm$ 0.006 \\
        TransMorph-diff & 83.51 $\pm$ 1.52 & 0.0066 $\pm$ 0.0073 & 2.35 $\pm$ 0.76 & 0.78 $\pm$ 0.27 & \textcolor{blue}{91.25 $\pm$ 2.07} & 0.071 $\pm$ 0.041 \\
        R2Net & 80.11 $\pm$ 1.53 & 0.0093 $\pm$ 0.0054 & 2.15 $\pm$ 0.49 & 1.02 $\pm$ 1.02 & 87.44 $\pm$ 1.15 & 0.081 $\pm$ 0.055 \\
        NODEO & 81.97 $\pm$ 1.39 & \textcolor{blue}{0.0024 $\pm$ 0.0010} & 2.18 $\pm$ 0.71 & 0.83 $\pm$ 0.24 & 87.68 $\pm$ 1.93 & \textcolor{blue}{0.008 $\pm$ 0.005} \\
        NePhi & 80.09 $\pm$ 1.17 & 0.0056 $\pm$ 0.0023 & 3.28 $\pm$ 0.44 & 0.90 $\pm$ 0.32 & 86.15 $\pm$ 1.90 & 0.022 $\pm$ 0.016 \\
        PULPo & 80.51 $\pm$ 0.99 & 0.0527 $\pm$ 0.0240 & 3.45 $\pm$ 0.98 & 0.85 $\pm$ 0.34 & 86.71 $\pm$ 1.86 & 0.077 $\pm$ 0.031 \\
        \midrule
        TransMorph & 84.11 $\pm$ 1.30 & 1.0665 $\pm$ 0.5631 & 2.26 $\pm$ 0.68 & 0.80 $\pm$ 0.26 & 91.14 $\pm$ 1.49 & 0.821 $\pm$ 0.252 \\
        TransMatch & 83.48 $\pm$ 1.80 & 0.1493 $\pm$ 0.0151 & 2.23 $\pm$ 0.66 & \textcolor{red}{0.75 $\pm$ 0.21} & 90.26 $\pm$ 2.11 & 0.741 $\pm$ 0.344 \\
        DiffuseMorph & 79.13 $\pm$ 2.03 & 1.1563 $\pm$ 0.2512 & 3.51 $\pm$ 0.77 & 0.85 $\pm$ 0.29 & 87.79 $\pm$ 1.91 & 0.336 $\pm$ 0.189 \\
        DiffuseReg & 80.12 $\pm$ 1.22 & 1.0877 $\pm$ 0.4312 & 3.73 $\pm$ 1.01 & 0.90 $\pm$ 0.33 & 87.94 $\pm$ 2.15 & 0.377 $\pm$ 0.130 \\
        HViT & \textcolor{red}{85.07 $\pm$ 1.05} & 0.4812 $\pm$ 0.0114 & \textcolor{red}{1.92 $\pm$ 0.46} & 0.83 $\pm$ 0.31 & \textcolor{red}{92.13 $\pm$ 1.39} & 0.631 $\pm$ 0.412 \\
        CorrMLP & 84.66 $\pm$ 1.24 & 0.4640 $\pm$ 0.1776 & 2.23 $\pm$ 0.62 & 0.76 $\pm$ 0.22 & 90.29 $\pm$ 1.84 & 0.608 $\pm$ 0.278 \\
        SACB-Net & 83.17 $\pm$ 1.18 & \textcolor{red}{0.0184 $\pm$ 0.0093} & 2.33 $\pm$ 0.48 & 0.80 $\pm$ 0.21 & 88.66 $\pm$ 2.13 & \textcolor{red}{0.182 $\pm$ 0.087} \\
        DGIR & 82.37 $\pm$ 1.33 & 0.0325 $\pm$ 0.0076 & 3.29 $\pm$ 0.91 & 0.85 $\pm$ 0.22 & 91.66 $\pm$ 1.79 & 0.392 $\pm$ 0.182 \\
        \midrule
        \textbf{TPFM-DIR} & \textbf{87.08 $\pm$ 1.07} & \textbf{0.0019 $\pm$ 0.0007} & \textbf{1.78 $\pm$ 0.44} & \textbf{0.75 $\pm$ 0.12} & \textbf{93.34 $\pm$ 1.83} & \textbf{0.006 $\pm$ 0.002} \\
        \midrule
        \midrule
        \multicolumn{7}{c}{\textbf{IXI}} \\
        \midrule
        \textbf{Method} 
        & \textbf{Dice $\uparrow$} 
        & $|J|_{<0}\% \downarrow$ 
        & \textbf{HD95} $\downarrow$ 
        & \textbf{ASSD} $\downarrow$ 
        & \textbf{SSIM} $\uparrow$
        & \textbf{SDLogJ} $\downarrow$ \\
        \midrule
        LDDMM & 67.90 $\pm$ 1.15 & 0.0034 $\pm$ 0.0021 & 4.11 $\pm$ 1.02 & 1.17 $\pm$ 0.55 & 81.1 $\pm$ 1.75 & 0.058 $\pm$ 0.020 \\
        \midrule
        CycleMorph & 74.67 $\pm$ 1.33 & 0.0224 $\pm$ 0.0110 & 3.34 $\pm$ 0.87 & 0.87 $\pm$ 0.38 & 84.75 $\pm$ 1.88 & 0.068 $\pm$ 0.032 \\
        GradICON & \textcolor{blue}{76.43 $\pm$ 1.28} & \textcolor{blue}{0.0018 $\pm$ 0.0022} & \textcolor{blue}{3.19 $\pm$ 0.55} & 0.87 $\pm$ 0.30 & 85.70 $\pm$ 1.73 & \textcolor{blue}{0.015 $\pm$ 0.006} \\
        TransMorph-diff & 75.98 $\pm$ 2.01 & 0.0142 $\pm$ 0.0098 & 3.81 $\pm$ 0.89 & 0.94 $\pm$ 0.45 & \textcolor{blue}{87.75 $\pm$ 1.92} & 0.102 $\pm$ 0.048 \\
        R2Net & 73.39 $\pm$ 1.77 & 0.0031 $\pm$ 0.0016 & 3.69 $\pm$ 0.66 & 1.20 $\pm$ 0.51 & 83.92 $\pm$ 1.97 & 0.042 $\pm$ 0.026 \\
        NODEO & 75.68 $\pm$ 1.82 & 0.0031 $\pm$ 0.0015 & 3.46 $\pm$ 0.84 & 0.94 $\pm$ 0.42 & 86.46 $\pm$ 1.38 & 0.017 $\pm$ 0.008 \\
        NePhi & 75.34 $\pm$ 1.78 & 0.0029 $\pm$ 0.0016 & 3.55 $\pm$ 0.55 & 0.93 $\pm$ 0.34 & 85.05 $\pm$ 1.83 & 0.018 $\pm$ 0.011 \\
        PULPo & 73.91 $\pm$ 2.11 & 0.0066 $\pm$ 0.0043 & 3.30 $\pm$ 0.39 & 0.90 $\pm$ 0.41 & 84.35 $\pm$ 2.18 & 0.028 $\pm$ 0.016 \\
        \midrule
        TransMorph & 77.72 $\pm$ 1.76 & 1.2574 $\pm$ 0.7574 & 3.71 $\pm$ 0.80 & 0.88 $\pm$ 0.32 & 89.04 $\pm$ 1.76 & 1.083 $\pm$ 0.742 \\
        TransMatch & 76.80 $\pm$ 1.85 & 0.0406 $\pm$ 0.0120 & 3.15 $\pm$ 0.41 & 0.84 $\pm$ 0.34 & 87.58 $\pm$ 1.74 & 0.580 $\pm$ 0.411 \\
        DiffuseMorph & 71.35 $\pm$ 2.05 & 1.6042 $\pm$ 0.7212 & 5.48 $\pm$ 1.45 & 1.45 $\pm$ 0.47 & 83.19 $\pm$ 2.76 & 0.938 $\pm$ 0.508 \\
        DiffuseReg & 72.13 $\pm$ 1.88 & 0.5279 $\pm$ 0.2834 & 5.04 $\pm$ 1.12 & 1.28 $\pm$ 0.43 & 84.13 $\pm$ 2.45 & 0.488 $\pm$ 0.276 \\
        HViT & \textcolor{red}{80.67 $\pm$ 1.67} & 0.5933 $\pm$ 0.1028 & \textcolor{red}{2.98 $\pm$ 0.46} & \textbf{0.78 $\pm$ 0.28} & \textcolor{red}{89.48 $\pm$ 1.92} & 0.679 $\pm$ 0.290 \\
        CorrMLP & 77.54 $\pm$ 1.68 & 0.3675 $\pm$ 0.2168 & 3.15 $\pm$ 0.39 & \textcolor{red}{0.81 $\pm$ 0.31} & 86.50 $\pm$ 1.64 & 0.547 $\pm$ 0.231 \\
        SACB-Net & 78.15 $\pm$ 1.87 & \textcolor{red}{0.0201 $\pm$ 0.0099} & 3.19 $\pm$ 0.51 & 0.84 $\pm$ 0.24 & 87.61 $\pm$ 1.89 & \textcolor{red}{0.144 $\pm$ 0.056} \\
        DGIR & 75.34 $\pm$ 1.55 & 0.0766 $\pm$ 0.0101 & 3.22 $\pm$ 0.72 & 0.89 $\pm$ 0.22 & 84.93 $\pm$ 1.91 & 0.193 $\pm$ 0.101 \\
        \midrule
        \textbf{TPFM-DIR} & \textbf{82.52 $\pm$ 0.98} & \textbf{0.0015 $\pm$ 0.0006} & \textbf{2.95 $\pm$ 0.68} & \textcolor{blue}{0.79 $\pm$ 0.17} & \textbf{90.88 $\pm$ 1.03} & \textbf{0.007 $\pm$ 0.003} \\
        \midrule
        \midrule
        \multicolumn{7}{c}{\textbf{LPBA40}} \\
        \midrule
        \textbf{Method} 
        & \textbf{Dice $\uparrow$} 
        & $|J|_{<0}\% \downarrow$ 
        & \textbf{HD95} $\downarrow$ 
        & \textbf{ASSD} $\downarrow$ 
        & \textbf{SSIM} $\uparrow$
        & \textbf{SDLogJ} $\downarrow$ \\
        \midrule
        LDDMM & 69.23 $\pm$ 1.78 & 0.0011 $\pm$ 0.0004 & 7.38 $\pm$ 1.11 & 1.86 $\pm$ 1.03 & 84.87 $\pm$ 2.44 & 0.027 $\pm$ 0.011 \\
        \midrule
        CycleMorph & 71.90 $\pm$ 1.41 & 0.0064 $\pm$ 0.0032 & 6.76 $\pm$ 0.94 & 1.28 $\pm$ 0.86 & 87.43 $\pm$ 2.18 & 0.042 $\pm$ 0.026 \\
        GradICON & 74.21 $\pm$ 1.26 & 0.0011 $\pm$ 0.0006 & 6.89 $\pm$ 0.89 & \textcolor{blue}{1.02 $\pm$ 0.76} & 91.32 $\pm$ 1.59 & \textcolor{blue}{0.022 $\pm$ 0.016} \\
        TransMorph-diff & 70.98 $\pm$ 1.52 & 0.0086 $\pm$ 0.0044 & 7.22 $\pm$ 1.03 & 1.18 $\pm$ 0.82 & 94.02 $\pm$ 1.45 & 0.041 $\pm$ 0.021 \\
        R2Net & 70.71 $\pm$ 1.49 & 0.0017 $\pm$ 0.0009 & 7.13 $\pm$ 1.01 & 1.88 $\pm$ 1.10 & 80.40 $\pm$ 2.66 & 0.029 $\pm$ 0.017 \\
        NODEO & 72.73 $\pm$ 1.25 & \textcolor{blue}{0.0008 $\pm$ 0.0004} & 7.25 $\pm$ 0.97 & 1.43 $\pm$ 0.53 & 91.57 $\pm$ 1.66 & 0.038 $\pm$ 0.023 \\
        NePhi & 72.39 $\pm$ 1.33 & 0.0024 $\pm$ 0.0016 & 6.75 $\pm$ 0.91 & 1.24 $\pm$ 0.55 & 85.28 $\pm$ 1.80 & 0.025 $\pm$ 0.017 \\
        PULPo & \textcolor{blue}{74.82 $\pm$ 1.11} & 0.0009 $\pm$ 0.0005 & \textcolor{blue}{6.22 $\pm$ 0.78} & 1.09 $\pm$ 0.43 & 86.35 $\pm$ 1.44 & 0.031 $\pm$ 0.019 \\
        \midrule
        TransMorph & 74.42 $\pm$ 1.33 & 0.5275 $\pm$ 0.1835 & 6.88 $\pm$ 0.85 & 1.78 $\pm$ 0.68 & \textbf{94.57 $\pm$ 1.86} & 0.217 $\pm$ 0.180 \\
        TransMatch & 74.13 $\pm$ 1.31 & 0.0251 $\pm$ 0.0116 & \textcolor{red}{6.53 $\pm$ 0.83} & 1.68 $\pm$ 0.93 & 92.47 $\pm$ 2.07 & 0.126 $\pm$ 0.091 \\
        DiffuseMorph & 73.11 $\pm$ 1.42 & 0.3610 $\pm$ 0.1429 & 7.21 $\pm$ 1.06 & 2.08 $\pm$ 1.12 & 84.50 $\pm$ 2.33 & 0.521 $\pm$ 0.374 \\
        DiffuseReg & 72.98 $\pm$ 1.46 & 0.0723 $\pm$ 0.0289 & 7.27 $\pm$ 1.04 & 2.08 $\pm$ 0.93 & 85.76 $\pm$ 2.11 & 0.370 $\pm$ 0.245 \\
        HViT & \textcolor{red}{76.57 $\pm$ 1.18} & 0.4428 $\pm$ 0.1655 & 6.74 $\pm$ 0.89 & 1.71 $\pm$ 0.54 & \textcolor{red}{94.10 $\pm$ 2.12} & 0.397 $\pm$ 0.122 \\
        CorrMLP & 75.72 $\pm$ 1.27 & 0.0145 $\pm$ 0.0068 & 6.55 $\pm$ 0.87 & 1.73 $\pm$ 0.78 & 90.39 $\pm$ 1.34 & 0.197 $\pm$ 0.068 \\
        SACB-Net & 73.91 $\pm$ 1.35 & \textcolor{red}{0.0099 $\pm$ 0.0041} & 6.70 $\pm$ 0.93 & \textcolor{red}{1.47 $\pm$ 0.38} & 93.49 $\pm$ 1.66 & \textcolor{red}{0.121 $\pm$ 0.092} \\
        DGIR & 71.99 $\pm$ 1.44 & 0.0189 $\pm$ 0.0073 & 7.01 $\pm$ 0.89 & 1.88 $\pm$ 1.12 & 88.98 $\pm$ 1.79 & 0.181 $\pm$ 0.160 \\
        \midrule
        \textbf{TPFM-DIR} & \textbf{77.87 $\pm$ 1.09} & \textbf{0.0005 $\pm$ 0.0002} & \textbf{5.85 $\pm$ 1.63} & \textbf{0.98 $\pm$ 0.15} & \textcolor{blue}{94.24 $\pm$ 1.82} & \textbf{0.021 $\pm$ 0.012} \\
        \bottomrule
    \end{tabular}}
    \label{tab:oasis_ixi_appendix}
\end{table*}

\begin{table*}[!t]
    \centering
    \caption{Quantitative comparison on the Mindboggle101, CANDI, and AbdomenCT datasets. 
    The best result per metric is shown in \textbf{bold}. 
    The second-best diffeomorphic and deformable results are highlighted in 
    \textcolor{blue}{blue} and \textcolor{red}{red}, respectively.}
    \resizebox{0.85\linewidth}{!}{
    \begin{tabular}{lcccccc}
        \toprule
        \multicolumn{7}{c}{\textbf{Mindboggle101}} \\
        \midrule
        \textbf{Method} 
        & \textbf{Dice $\uparrow$} 
        & $|J|_{<0}\% \downarrow$ 
        & \textbf{HD95} $\downarrow$ 
        & \textbf{ASSD} $\downarrow$ 
        & \textbf{SSIM} $\uparrow$
        & \textbf{SDLogJ} $\downarrow$ \\
        \midrule
        LDDMM & 64.57 $\pm$ 2.08 & 0.0061 $\pm$ 0.0039 & 7.79 $\pm$ 0.84 & 1.97 $\pm$ 0.78 & 81.48 $\pm$ 2.42 & 0.028 $\pm$ 0.019 \\
        \midrule
        CycleMorph & 70.64 $\pm$ 1.48 & 0.1282 $\pm$ 0.0873 & 7.75 $\pm$ 0.98 & 1.56 $\pm$ 1.02 & 86.83 $\pm$ 2.44 & 0.145 $\pm$ 0.033 \\
        GradICON & \textcolor{blue}{70.78 $\pm$ 1.52} & 0.0117 $\pm$ 0.0043 & 7.59 $\pm$ 1.12 & 1.57 $\pm$ 0.72 & 87.51 $\pm$ 1.76 & \textcolor{blue}{0.034 $\pm$ 0.023} \\
        TransMorph-diff & 69.45 $\pm$ 1.66 & 0.0364 $\pm$ 0.0157 & 7.89 $\pm$ 1.05 & \textcolor{blue}{1.47 $\pm$ 0.83} & \textcolor{blue}{92.76 $\pm$ 1.49} & 0.081 $\pm$ 0.048 \\
        R2Net & 68.34 $\pm$ 1.69 & 0.0086 $\pm$ 0.0031 & 8.97 $\pm$ 1.44 & 1.84 $\pm$ 1.03 & 82.61 $\pm$ 2.18 & 0.087 $\pm$ 0.048 \\
        NODEO & 70.18 $\pm$ 1.58 & 0.0023 $\pm$ 0.0011 & 7.70 $\pm$ 0.96 & 1.66 $\pm$ 0.67 & 86.72 $\pm$ 1.83 & 0.058 $\pm$ 0.038 \\
        NePhi & 67.97 $\pm$ 1.91 & \textcolor{blue}{0.0019 $\pm$ 0.0005} & 8.20 $\pm$ 1.23 & 1.74 $\pm$ 0.84 & 84.47 $\pm$ 1.77 & 0.047 $\pm$ 0.025 \\
        PULPo & 68.69 $\pm$ 1.74 & 0.0287 $\pm$ 0.0104 & 7.92 $\pm$ 1.07 & 1.61 $\pm$ 1.06 & 84.23 $\pm$ 2.24 & 0.031 $\pm$ 0.027 \\
        \midrule
        TransMorph & 71.46 $\pm$ 1.63 & 1.4883 $\pm$ 0.5328 & \textbf{7.27 $\pm$ 0.92} & 1.45 $\pm$ 0.72 & \textbf{93.70 $\pm$ 2.08} & 0.337 $\pm$ 0.196 \\
        TransMatch & 71.78 $\pm$ 1.49 & 0.2483 $\pm$ 0.1020 & 7.63 $\pm$ 1.08 & 1.53 $\pm$ 1.14 & 89.73 $\pm$ 2.13 & 0.561 $\pm$ 0.399 \\
        DiffuseMorph & 68.22 $\pm$ 1.77 & 1.7413 $\pm$ 0.6521 & 8.43 $\pm$ 1.30 & 1.65 $\pm$ 0.75 & 84.62 $\pm$ 1.76 & 0.535 $\pm$ 0.382 \\
        DiffuseReg & 71.88 $\pm$ 1.54 & 1.0823 $\pm$ 0.4882 & 7.84 $\pm$ 1.03 & 1.56 $\pm$ 1.14 & 90.20 $\pm$ 2.54 & 0.224 $\pm$ 0.188 \\
        HViT & 71.80 $\pm$ 1.72 & 1.4051 $\pm$ 0.6153 & 7.56 $\pm$ 1.11 & \textcolor{red}{1.51 $\pm$ 0.58} & \textcolor{red}{92.04 $\pm$ 1.49} & 0.352 $\pm$ 0.287 \\
        CorrMLP & 71.92 $\pm$ 1.68 & 0.1895 $\pm$ 0.0782 & 7.65 $\pm$ 1.06 & 1.51 $\pm$ 0.67 & 89.48 $\pm$ 1.93 & 0.331 $\pm$ 0.214 \\
        SACB-Net & \textcolor{red}{72.75 $\pm$ 1.57} & 0.3542 $\pm$ 0.1433 & \textcolor{red}{7.46 $\pm$ 0.97} & 1.52 $\pm$ 0.88 & 91.83 $\pm$ 2.11 & 0.214 $\pm$ 0.197 \\
        DGIR & 68.12 $\pm$ 1.93 & \textcolor{red}{0.0674 $\pm$ 0.0213} & 7.94 $\pm$ 1.08 & 1.77 $\pm$ 1.23 & 86.17 $\pm$ 2.13 & \textcolor{red}{0.093 $\pm$ 0.077} \\
        \midrule
        \textbf{TPFM-DIR} & \textbf{74.79 $\pm$ 0.98} & \textbf{0.0012 $\pm$ 0.0017} & \textcolor{blue}{7.53 $\pm$ 1.86} & \textbf{1.34 $\pm$ 0.63} & 91.87 $\pm$ 1.93 & \textbf{0.022 $\pm$ 0.013} \\
        \midrule
        \midrule
        \multicolumn{7}{c}{\textbf{CANDI}} \\
        \midrule
        \textbf{Method} 
        & \textbf{Dice $\uparrow$} 
        & $|J|_{<0}\% \downarrow$ 
        & \textbf{HD95} $\downarrow$ 
        & \textbf{ASSD} $\downarrow$ 
        & \textbf{SSIM} $\uparrow$
        & \textbf{SDLogJ} $\downarrow$ \\
        \midrule
        LDDMM & 78.52 $\pm$ 1.66 & 0.0017 $\pm$ 0.0007 & 2.44 $\pm$ 0.31 & 0.92 $\pm$ 0.25 & 94.49 $\pm$ 1.76 & 0.023 $\pm$ 0.010 \\
        \midrule
        CycleMorph & 81.37 $\pm$ 1.25 & 0.0027 $\pm$ 0.0013 & 2.23 $\pm$ 0.29 & 0.96 $\pm$ 0.30 & 95.29 $\pm$ 1.66 & 0.075 $\pm$ 0.031 \\
        GradICON & \textcolor{blue}{83.39 $\pm$ 1.09} & 0.0015 $\pm$ 0.0008 & 2.14 $\pm$ 0.27 & \textbf{0.72 $\pm$ 0.16} & 96.35 $\pm$ 1.55 & \textcolor{blue}{0.018 $\pm$ 0.011} \\
        TransMorph-diff & 83.20 $\pm$ 1.14 & 0.0026 $\pm$ 0.0012 & \textcolor{blue}{2.06 $\pm$ 0.26} & 0.80 $\pm$ 0.48 & 96.99 $\pm$ 1.83 & 0.052 $\pm$ 0.038 \\
        R2Net & 78.79 $\pm$ 1.49 & 0.0032 $\pm$ 0.0016 & 2.78 $\pm$ 0.32 & 0.87 $\pm$ 0.29 & 91.12 $\pm$ 1.28 & 0.023 $\pm$ 0.019 \\
        NODEO & 80.60 $\pm$ 1.26 & 0.0012 $\pm$ 0.0007 & 2.46 $\pm$ 0.29 & 0.86 $\pm$ 0.27 & 94.29 $\pm$ 1.64 & 0.031 $\pm$ 0.016 \\
        NePhi & 79.15 $\pm$ 1.46 & \textcolor{blue}{0.0009 $\pm$ 0.0005} & 2.36 $\pm$ 0.30 & 0.83 $\pm$ 0.27 & 92.07 $\pm$ 1.74 & \textbf{0.016 $\pm$ 0.011} \\
        PULPo & 78.38 $\pm$ 1.57 & 0.0136 $\pm$ 0.0047 & 2.89 $\pm$ 0.33 & 0.93 $\pm$ 0.36 & 91.62 $\pm$ 1.33 & 0.048 $\pm$ 0.032 \\
        \midrule
        TransMorph & 83.48 $\pm$ 1.10 & 0.1059 $\pm$ 0.0388 & 2.13 $\pm$ 0.22 & 0.87 $\pm$ 0.38 & \textcolor{red}{97.44 $\pm$ 1.74} & 0.403 $\pm$ 0.213 \\
        TransMatch & 82.60 $\pm$ 1.18 & 0.0160 $\pm$ 0.0061 & 2.28 $\pm$ 0.25 & 0.87 $\pm$ 0.47 & 96.48 $\pm$ 1.59 & 0.224 $\pm$ 0.155 \\
        DiffuseMorph & 77.22 $\pm$ 1.51 & 0.1277 $\pm$ 0.0412 & 2.53 $\pm$ 0.30 & 1.12 $\pm$ 0.55 & 92.64 $\pm$ 1.83 & 0.488 $\pm$ 0.201 \\
        DiffuseReg & 79.31 $\pm$ 1.38 & 0.0322 $\pm$ 0.0125 & 2.31 $\pm$ 0.26 & 1.03 $\pm$ 0.44 & 93.88 $\pm$ 2.01 & 0.388 $\pm$ 0.251 \\
        HViT & \textcolor{red}{83.67 $\pm$ 1.07} & 0.1368 $\pm$ 0.0435 & \textcolor{red}{2.09 $\pm$ 0.21} & 0.87 $\pm$ 0.31 & \textbf{97.63 $\pm$ 1.58} & 0.251 $\pm$ 0.187 \\
        CorrMLP & 81.96 $\pm$ 1.25 & 0.0098 $\pm$ 0.0043 & 2.15 $\pm$ 0.26 & \textcolor{red}{0.82 $\pm$ 0.18} & 95.68 $\pm$ 1.41 & 0.278 $\pm$ 0.140 \\
        SACB-Net & 82.75 $\pm$ 1.22 & \textcolor{red}{0.0047 $\pm$ 0.0021} & 2.19 $\pm$ 0.27 & 0.88 $\pm$ 0.39 & 96.22 $\pm$ 1.78 & 0.298 $\pm$ 0.151 \\
        DGIR & 80.13 $\pm$ 1.39 & 0.0077 $\pm$ 0.0028 & 2.74 $\pm$ 0.41 & 0.94 $\pm$ 0.37 & 93.48 $\pm$ 1.70 & \textcolor{red}{0.177 $\pm$ 0.094} \\
        \midrule
        \textbf{TPFM-DIR} & \textbf{84.77 $\pm$ 1.37} & \textbf{0.0001 $\pm$ 0.0001} & \textbf{1.95 $\pm$ 0.49} & \textcolor{blue}{0.76 $\pm$ 0.12} & \textcolor{blue}{97.28 $\pm$ 1.29} & 0.019 $\pm$ 0.008 \\
        \midrule
        \midrule
        \multicolumn{7}{c}{\textbf{AbdomenCT}} \\
        \midrule
        \textbf{Method} 
        & \textbf{Dice $\uparrow$} 
        & $|J|_{<0}\% \downarrow$ 
        & \textbf{HD95} $\downarrow$ 
        & \textbf{ASSD} $\downarrow$ 
        & \textbf{SSIM} $\uparrow$
        & \textbf{SDLogJ} $\downarrow$ \\
        \midrule
        LDDMM & 37.36 $\pm$ 2.71 & 0.0112 $\pm$ 0.0072 & 12.61 $\pm$ 1.46 & 4.73 $\pm$ 2.45 & 53.83 $\pm$ 3.17 & 0.189 $\pm$ 0.124 \\
        \midrule
        CycleMorph & 42.56 $\pm$ 2.65 & 0.0053 $\pm$ 0.0025 & 13.04 $\pm$ 1.28 & 4.77 $\pm$ 1.39 & 52.68 $\pm$ 3.02 & 0.108 $\pm$ 0.058 \\
        GradICON & 44.02 $\pm$ 2.81 & 0.0009 $\pm$ 0.0004 & \textcolor{blue}{12.21 $\pm$ 1.33} & 4.02 $\pm$ 1.72 & 54.86 $\pm$ 3.27 & \textcolor{blue}{0.026 $\pm$ 0.011} \\
        TransMorph-diff & 41.41 $\pm$ 2.72 & 0.0034 $\pm$ 0.0017 & 12.39 $\pm$ 1.25 & 4.29 $\pm$ 1.47 & 62.61 $\pm$ 2.98 & 0.029 $\pm$ 0.016 \\
        R2Net & 41.99 $\pm$ 2.89 & 0.0010 $\pm$ 0.0005 & 13.05 $\pm$ 1.32 & 4.93 $\pm$ 2.08 & 51.52 $\pm$ 3.36 & 0.071 $\pm$ 0.058 \\
        NODEO & 40.37 $\pm$ 2.96 & \textcolor{blue}{0.0007 $\pm$ 0.0004} & 13.16 $\pm$ 1.38 & 4.33 $\pm$ 1.83 & 52.45 $\pm$ 3.42 & 0.039 $\pm$ 0.019 \\
        NePhi & \textcolor{blue}{45.32 $\pm$ 2.24} & 0.0008 $\pm$ 0.0003 & 12.48 $\pm$ 1.19 & 3.91 $\pm$ 1.37 & 66.34 $\pm$ 2.54 & 0.031 $\pm$ 0.018 \\
        PULPo & 44.53 $\pm$ 2.13 & 0.0819 $\pm$ 0.0173 & 11.89 $\pm$ 1.86 & \textcolor{blue}{3.51 $\pm$ 1.14} & 53.89 $\pm$ 2.81 & 0.038 $\pm$ 0.015 \\
        \midrule
        TransMorph & 46.66 $\pm$ 2.48 & 3.1310 $\pm$ 0.8113 & 12.61 $\pm$ 1.28 & 4.31 $\pm$ 1.82 & \textbf{71.84 $\pm$ 2.54} & 0.580 $\pm$ 0.279 \\
        TransMatch & 43.64 $\pm$ 2.35 & 2.6378 $\pm$ 0.6922 & 12.50 $\pm$ 1.22 & 3.87 $\pm$ 1.86 & \textcolor{red}{71.47 $\pm$ 2.33} & 0.917 $\pm$ 0.688 \\
        DiffuseMorph & 39.50 $\pm$ 2.81 & 0.0292 $\pm$ 0.0113 & 12.88 $\pm$ 1.34 & 4.74 $\pm$ 1.49 & 53.02 $\pm$ 3.09 & 0.320 $\pm$ 0.218 \\
        DiffuseReg & 45.91 $\pm$ 2.56 & 1.8106 $\pm$ 0.7348 & 12.55 $\pm$ 1.25 & 4.32 $\pm$ 1.67 & 56.87 $\pm$ 2.88 & \textcolor{red}{0.249 $\pm$ 0.170} \\
        HViT & 47.66 $\pm$ 2.41 & 0.9601 $\pm$ 0.4175 & \textcolor{red}{10.11 $\pm$ 1.10} & 3.88 $\pm$ 1.25 & 66.90 $\pm$ 2.74 & 0.643 $\pm$ 0.392 \\
        CorrMLP & 50.28 $\pm$ 2.18 & 0.1656 $\pm$ 0.0643 & 11.47 $\pm$ 1.16 & \textcolor{red}{3.01 $\pm$ 1.83} & 62.66 $\pm$ 2.59 & 0.412 $\pm$ 0.288 \\
        SACB-Net & \textcolor{red}{53.38 $\pm$ 2.63} & 0.9348 $\pm$ 0.1135 & 13.09 $\pm$ 1.41 & 3.66 $\pm$ 1.76 & 66.73 $\pm$ 2.96 & 0.426 $\pm$ 0.344 \\
        DGIR & 42.08 $\pm$ 2.23 & \textcolor{red}{0.0277 $\pm$ 0.0114} & 12.24 $\pm$ 1.80 & 4.65 $\pm$ 1.59 & 64.48 $\pm$ 2.71 & 0.299 $\pm$ 0.170 \\
        \midrule
        \textbf{TPFM-DIR} & \textbf{54.54 $\pm$ 1.86} & \textbf{0.0002 $\pm$ 0.0001} & \textbf{9.81 $\pm$ 1.13} & \textbf{2.84 $\pm$ 0.98} & \textcolor{blue}{68.94 $\pm$ 1.77} & \textbf{0.025 $\pm$ 0.015} \\
        \bottomrule
    \end{tabular}}
    \label{tab:candi_mindboggle_appendix}
\end{table*}

\begin{table}[!ht]
    \centering
    \caption{Quantitative comparison on the LungCT dataset. The best result per metric is shown in \textbf{bold}. The second-best diffeomorphic and deformable results are highlighted in \textcolor{blue}{blue} and \textcolor{red}{red}, respectively.}
    \resizebox{0.6\linewidth}{!}{
    \begin{tabular}{lcccc}
        \toprule
        \textbf{Method} 
        & \textbf{TRE} $\downarrow$
        & $|J|_{<0}\% \downarrow$ 
        & \textbf{SSIM} $\uparrow$
        & \textbf{SDLogJ} $\downarrow$ \\
        \midrule
        LDDMM & 3.09 $\pm$ 0.26 & 0.0033 $\pm$ 0.0015 & 61.14 $\pm$ 3.11 & 0.015 $\pm$ 0.010 \\
        \midrule
        CycleMorph & 3.04 $\pm$ 0.19 & 0.7180 $\pm$ 0.2117 & 56.99 $\pm$ 3.44 & 0.046 $\pm$ 0.013 \\
        GradICON & \textcolor{blue}{2.64 $\pm$ 0.17} & \textcolor{blue}{0.0009 $\pm$ 0.0004} & 54.06 $\pm$ 2.86 & \textcolor{blue}{0.004 $\pm$ 0.004} \\
        TransMorph-diff & 2.89 $\pm$ 0.18 & 0.0042 $\pm$ 0.0015 & 66.85 $\pm$ 2.95 & 0.006 $\pm$ 0.007 \\
        R2Net & 2.99 $\pm$ 0.19 & 0.0211 $\pm$ 0.0068 & 68.06 $\pm$ 2.81 & 0.013 $\pm$ 0.004 \\
        NODEO & 2.87 $\pm$ 0.18 & 0.0038 $\pm$ 0.0013 & \textcolor{blue}{70.80 $\pm$ 2.74} & 0.012 $\pm$ 0.008 \\
        NePhi & 3.12 $\pm$ 0.21 & 0.0093 $\pm$ 0.0027 & 56.66 $\pm$ 3.33 & 0.004 $\pm$ 0.001 \\
        \midrule
        TransMorph & 2.85 $\pm$ 0.17 & 0.4767 $\pm$ 0.1546 & 68.82 $\pm$ 2.65 & 0.195 $\pm$ 0.050 \\
        TransMatch & 3.02 $\pm$ 0.19 & 0.1343 $\pm$ 0.0482 & 67.47 $\pm$ 2.89 & 0.088 $\pm$ 0.032 \\
        DiffuseMorph & 3.38 $\pm$ 0.21 & 0.0794 $\pm$ 0.0195 & 62.98 $\pm$ 2.92 & 0.035 $\pm$ 0.029 \\
        DiffuseReg & 3.26 $\pm$ 0.20 & 0.7765 $\pm$ 0.2337 & 65.75 $\pm$ 2.71 & 0.037 $\pm$ 0.011 \\
        HViT & 3.22 $\pm$ 0.20 & 0.1791 $\pm$ 0.0554 & 66.72 $\pm$ 3.41 & 0.027 $\pm$ 0.019 \\
        CorrMLP & \textcolor{red}{2.48 $\pm$ 0.16} & 0.2673 $\pm$ 0.0711 & 65.57 $\pm$ 3.08 & 0.073 $\pm$ 0.049 \\
        SACB-Net & 3.01 $\pm$ 0.19 & 0.9824 $\pm$ 0.2761 & \textcolor{red}{70.06 $\pm$ 2.77} & 0.059 $\pm$ 0.015 \\
        DGIR & 3.24 $\pm$ 0.42 & \textcolor{red}{0.0711 $\pm$ 0.0428} & 67.72 $\pm$ 2.41 & \textcolor{red}{0.025 $\pm$ 0.021} \\
        \midrule
        \textbf{TPFM-DIR} & \textbf{2.19 $\pm$ 0.14} & \textbf{0.0 $\pm$ 0.0} & \textbf{71.18 $\pm$ 2.35} & \textbf{0.002 $\pm$ 0.001} \\
        \bottomrule
    \end{tabular}}
    \label{tab:lungct_only}
\end{table}

\begin{table*}[!t]
    \centering
    \caption{Quantitative comparison on the ACDC and CAMUS datasets. The best result per metric is shown in \textbf{bold}. The second-best result is highlighted in \textcolor{red}{red}.}
    \resizebox{0.85\linewidth}{!}{
    \begin{tabular}{lcccccc}
        \toprule
        \multicolumn{7}{c}{\textbf{ACDC}} \\
        \midrule
        \textbf{Method} 
        & \textbf{Dice $\uparrow$} 
        & $|J|_{<0}\% \downarrow$ 
        & \textbf{HD95} $\downarrow$ 
        & \textbf{ASSD} $\downarrow$ 
        & \textbf{SSIM} $\uparrow$
        & \textbf{SDLogJ} $\downarrow$ \\
        \midrule
        LDDMM & 78.17 $\pm$ 1.52 & 0.0062 $\pm$ 0.0044 & \textcolor{red}{5.43 $\pm$ 0.63} & \textcolor{red}{2.02 $\pm$ 0.96} & 86.60 $\pm$ 1.87 & 0.197 $\pm$ 0.131 \\
        \midrule
        AdaCS & \textcolor{red}{85.46 $\pm$ 1.12} & 0.1493 $\pm$ 0.0185 & 5.56 $\pm$ 0.34 & 2.12 $\pm$ 1.20 & 84.46 $\pm$ 1.47 & 0.326 $\pm$ 0.131 \\
        TransMorph-diff & 82.13 $\pm$ 1.42 & \textcolor{red}{0.0012 $\pm$ 0.0005} & 6.55 $\pm$ 0.43 & 2.67 $\pm$ 1.36 & 79.56 $\pm$ 1.68 & \textbf{0.162 $\pm$ 0.069} \\
        TransMorph & 84.47 $\pm$ 1.18 & 0.0418 $\pm$ 0.0087 & 6.24 $\pm$ 0.36 & 2.54 $\pm$ 1.22 & 80.08 $\pm$ 1.41 & 0.575 $\pm$ 0.094 \\
        DiffuseMorph & 81.90 $\pm$ 1.35 & 0.0383 $\pm$ 0.0091 & 7.34 $\pm$ 0.47 & 2.81 $\pm$ 1.48 & 79.81 $\pm$ 1.62 & 0.251 $\pm$ 0.092 \\
        DiffuseReg & 83.31 $\pm$ 1.21 & 0.2345 $\pm$ 0.0228 & 6.17 $\pm$ 0.39 & 2.35 $\pm$ 1.31 & 82.13 $\pm$ 1.54 & 0.398 $\pm$ 0.107 \\
        DGIR & 82.39 $\pm$ 2.97 & 0.0059 $\pm$ 0.0125 & 7.43 $\pm$ 1.08 & 2.26 $\pm$ 0.79 & \textbf{92.91 $\pm$ 2.91} & 0.241 $\pm$ 0.021 \\
        \midrule
        \textbf{TPFM-DIR} & \textbf{87.42 $\pm$ 0.99} & \textbf{0.0002 $\pm$ 0.0001} & \textbf{5.36 $\pm$ 0.35} & \textbf{1.91 $\pm$ 1.07} & \textcolor{red}{87.17 $\pm$ 1.21} & \textcolor{red}{0.182 $\pm$ 0.023} \\
        \midrule
        \midrule
        \multicolumn{7}{c}{\textbf{CAMUS}} \\
        \midrule
        \textbf{Method} 
        & \textbf{Dice $\uparrow$} 
        & $|J|_{<0}\% \downarrow$ 
        & \textbf{HD95} $\downarrow$ 
        & \textbf{ASSD} $\downarrow$ 
        & \textbf{SSIM} $\uparrow$
        & \textbf{SDLogJ} $\downarrow$ \\
        \midrule
        LDDMM & 79.30 $\pm$ 5.21 & 0.0019 $\pm$ 0.0013 & 6.25 $\pm$ 1.35 & 2.61 $\pm$ 0.92 & 76.37 $\pm$ 1.94 & 0.195 $\pm$ 0.031 \\
        \midrule
        AdaCS & \textcolor{red}{82.67 $\pm$ 5.14} & 0.0029 $\pm$ 0.0130 & \textcolor{red}{5.23 $\pm$ 1.95} & \textcolor{red}{2.11 $\pm$ 0.76} & 77.18 $\pm$ 1.85 & 0.193 $\pm$ 0.033 \\
        TransMorph-diff & 78.63 $\pm$ 5.54 & \textcolor{red}{0.0011 $\pm$ 0.0016} & 10.15 $\pm$ 3.16 & 4.04 $\pm$ 1.12 & 77.57 $\pm$ 3.12 & \textcolor{red}{0.166 $\pm$ 0.028} \\
        TransMorph & 80.53 $\pm$ 5.93 & 0.1017 $\pm$ 0.1376 & 9.10 $\pm$ 3.16 & 3.68 $\pm$ 1.21 & \textcolor{red}{78.04 $\pm$ 3.32} & 0.312 $\pm$ 0.075 \\
        DiffuseMorph & 76.43 $\pm$ 6.72 & 0.0932 $\pm$ 0.1329 & 9.77 $\pm$ 3.18 & 4.28 $\pm$ 1.27 & 77.68 $\pm$ 3.57 & 0.393 $\pm$ 0.087 \\
        DiffuseReg & 77.29 $\pm$ 6.19 & 0.4026 $\pm$ 0.1402 & 10.60 $\pm$ 3.04 & 4.22 $\pm$ 1.23 & 77.75 $\pm$ 3.22 & 0.395 $\pm$ 0.062 \\
        DGIR & 81.83 $\pm$ 6.27 & 0.0315 $\pm$ 0.0282 & 6.71 $\pm$ 2.02 & 2.78 $\pm$ 0.82 & 76.91 $\pm$ 2.95 & 0.301 $\pm$ 0.066 \\
        \midrule
        \textbf{TPFM-DIR} & \textbf{85.12 $\pm$ 2.12} & \textbf{0.0 $\pm$ 0.0} & \textbf{4.72 $\pm$ 2.19} & \textbf{2.06 $\pm$ 0.94} & \textbf{79.65 $\pm$ 1.16} & \textbf{0.162 $\pm$ 0.020} \\
        \bottomrule
    \end{tabular}}
    \label{tab:acdc_camus_appendix}
\end{table*}

\section{Visual Results}
\label{app:visuals}

In this section we present the visual results on the datasets absent in the main paper. These include LPBA40 dataset on which we illustrate the forward/backward warping performance of TPFM-DIR (\cref{fig:lpba_appendix}), and IXI and Mindboggle101 dataset on which we provide a comparison with other diffeomorphic and deformbale methods (\cref{fig:ixi_appendix,fig:mindboggle_appendix}). Furthermore, we provide Dice score distribution across different anatomical regions for OASIS, IXI, CANDI, and ACDC datasets in \cref{fig:boxplots}.

\begin{figure}[!ht]
    \centering
    \includegraphics[width=\linewidth]{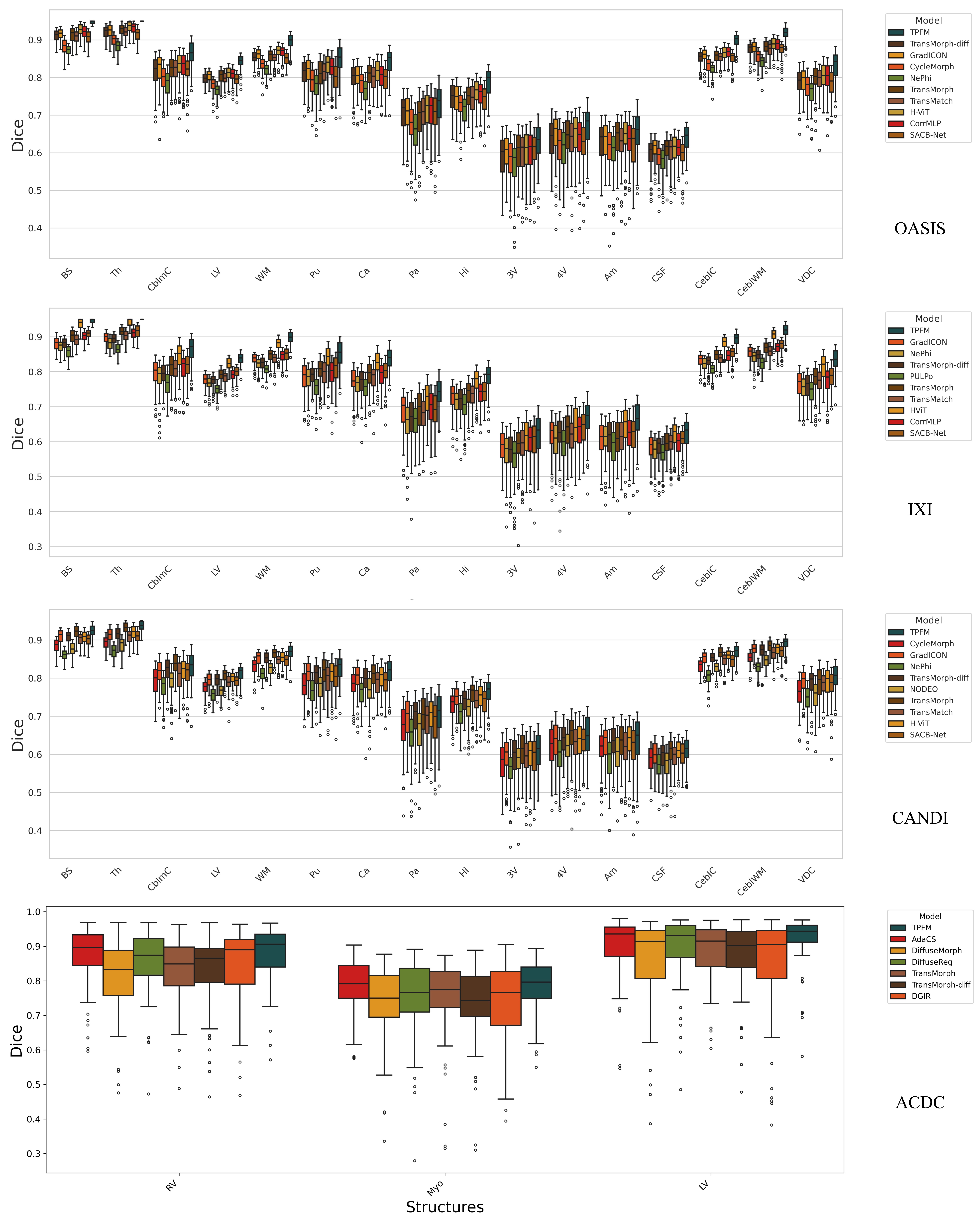}
    \caption{Boxplots of Dice performance across different anatomical regions over OASIS, IXI, CANDI, and ACDC datasets.}
    \label{fig:boxplots}
\end{figure}

\begin{figure}[!ht]
    \centering
    \includegraphics[width=\linewidth]{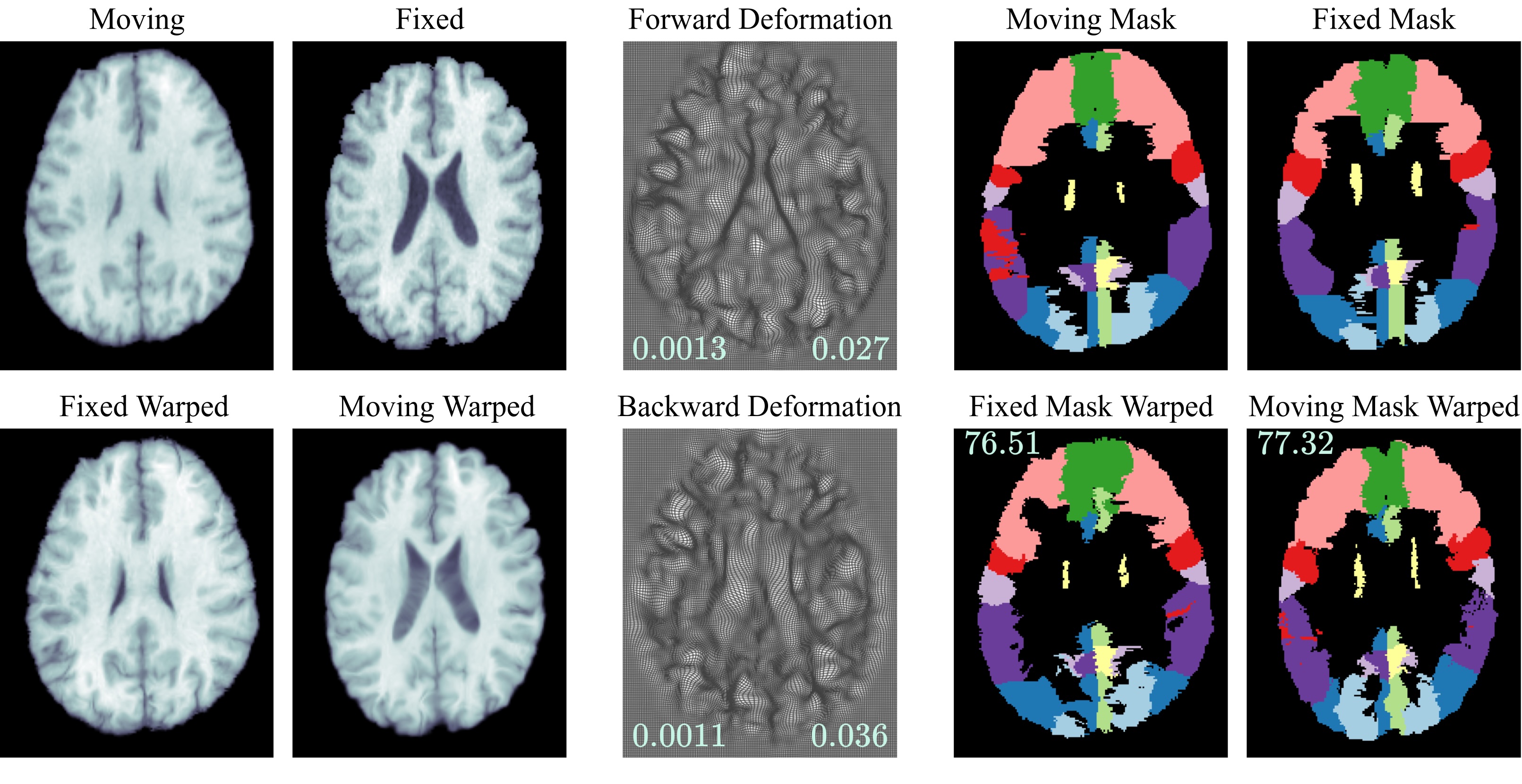}
    \caption{Visual result on forward and backward warping using TPFM-DIR on LPBA40. Forward and backward deformations warp the moving image towards the fixed image and vice versa. SDLogJ and $|J|_{<0}\%$ are reported at the bottom right and bottom left corners of the deformation grids. The forward and backward Dice scores are reported on the top left corners of the warped masks.}
    \label{fig:lpba_appendix}
\end{figure}

\begin{figure}[!ht]
    \centering
    \includegraphics[width=\linewidth]{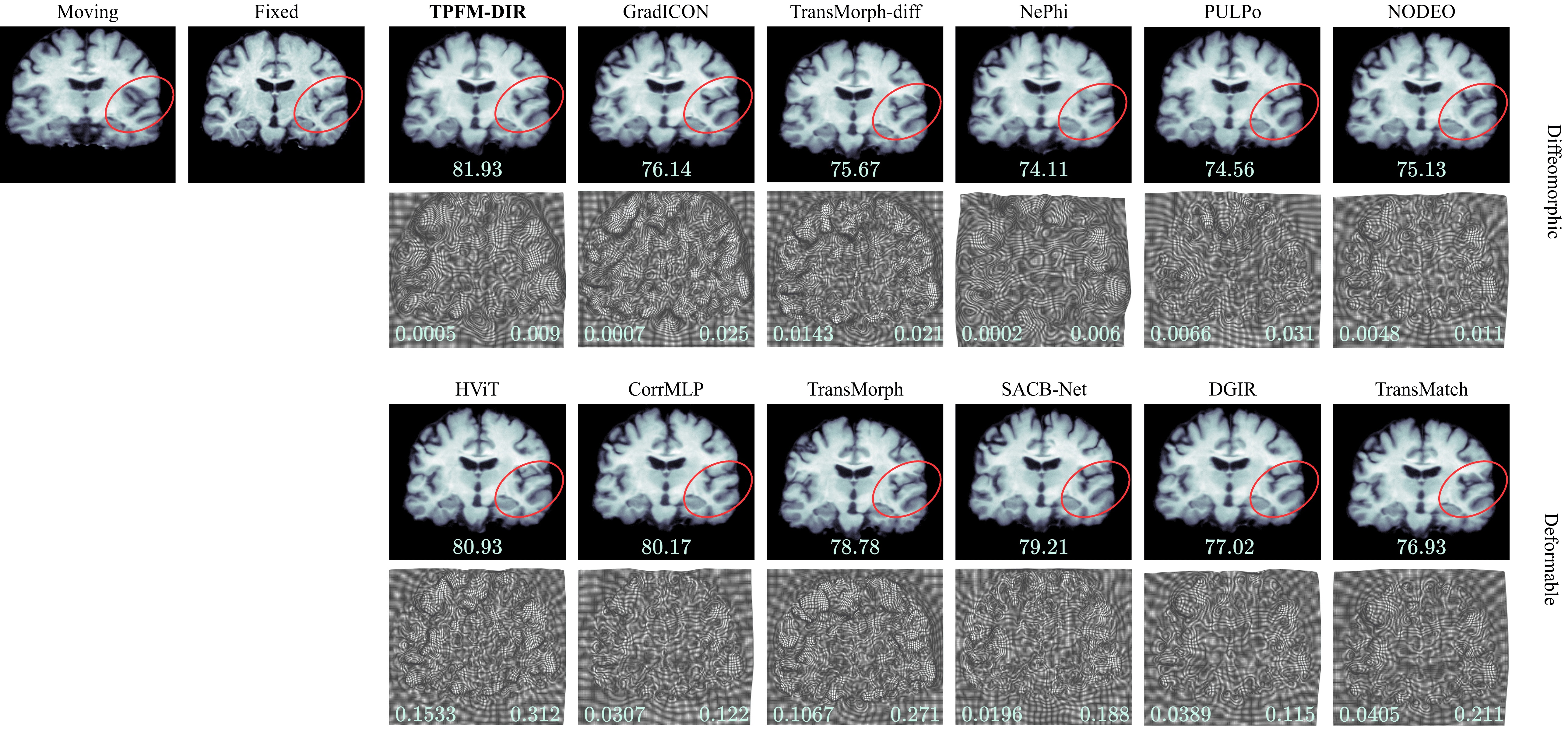}
    \caption{Visual results on IXI. Dice score is reported at the bottom of each warped image. SDLogJ and $|J|_{<0}\%$ are reported at the bottom right and bottom left corners of the deformation grids, respectively.}
    \label{fig:ixi_appendix}
\end{figure}

\begin{figure}[!ht]
    \centering
    \includegraphics[width=\linewidth]{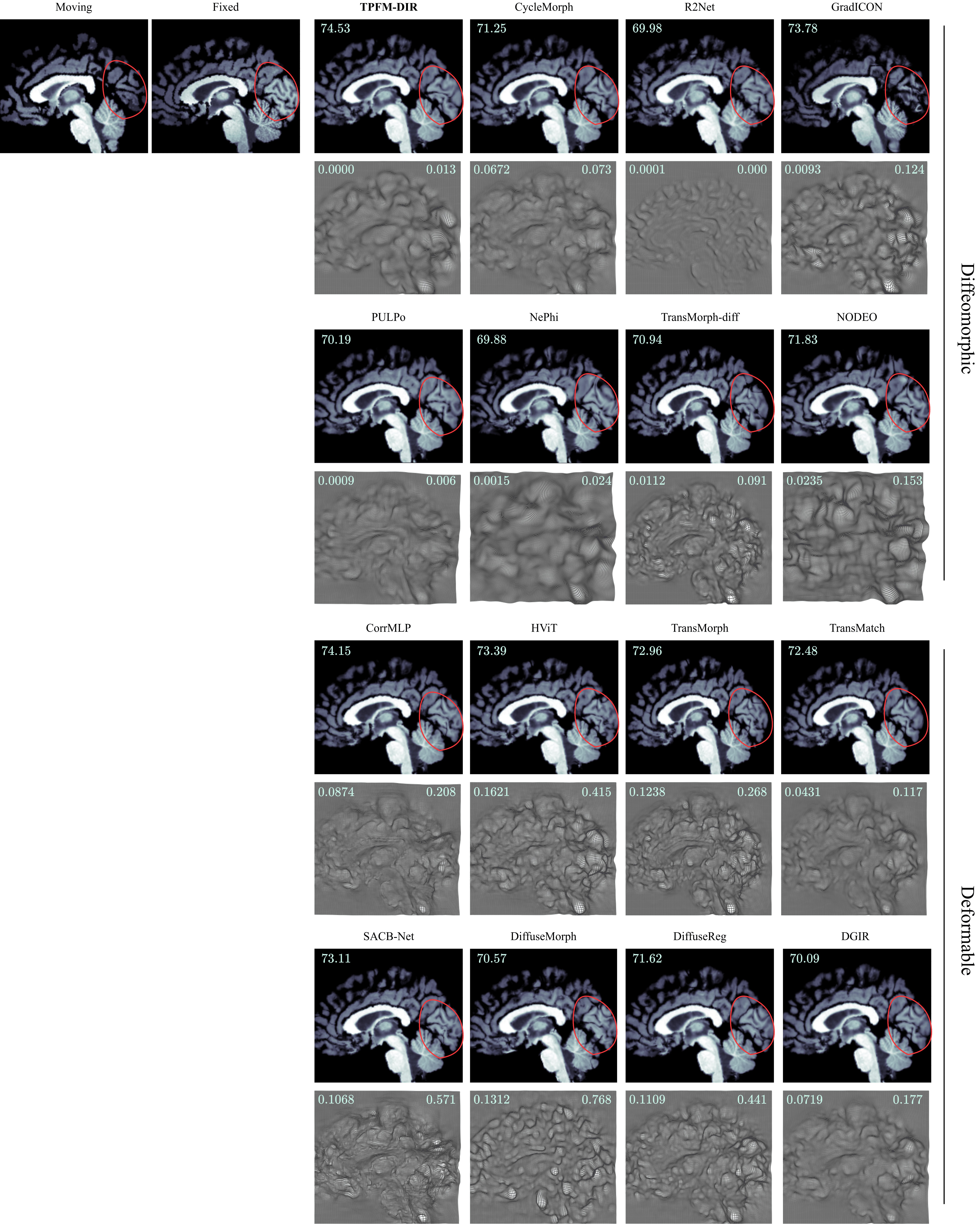}
    \caption{Visual results on Mindboggle101. Dice score is reported at the top left corner of each warped image. SDLogJ and $|J|_{<0}\%$ are reported at the top right and top left corners of the deformation grids, respectively.}
    \label{fig:mindboggle_appendix}
\end{figure}

\end{document}